\documentclass{article}

\usepackage{microtype}
\usepackage{graphicx}
\usepackage{subfigure}
\usepackage{caption}
\usepackage{booktabs} % for professional tables
\usepackage{multirow}
\usepackage{amsmath}

 \usepackage{hyperref}

\long\def\comment#1{}

\usepackage[ruled,vlined]{algorithm2e}

\usepackage[arxiv]{icml2021}

\usepackage{xcolor}
\usepackage{amssymb}
\usepackage{graphicx}
\usepackage{bbm}
\usepackage{mathrsfs}

\usepackage{graphicx}
\usepackage{caption}
\usepackage{float}

\DeclareMathOperator{\Prob}{\mathbb{P}}

\DeclareMathOperator*{\argmax}{arg\,max}

\newcommand{\norm}[1]{\left\lVert#1\right\rVert}
\newenvironment{proof}{\paragraph{Proof:}}{\hfill$\square$}

\newtheorem{prop}{Proposition}
\newtheorem{lemma}{Lemma}

\newtheorem{theorem}{Theorem}

\numberwithin{theorem}{section}
\numberwithin{defn}{section}
\numberwithin{prop}{section}
\numberwithin{lemma}{section}
\numberwithin{cor}{section}

\icmltitlerunning{Scalable nonparametric Bayesian learning for 
heterogeneous and dynamic velocity fields}

\begin{document}

\twocolumn[
\icmltitle{Scalable nonparametric Bayesian learning for \\
heterogeneous and dynamic velocity fields
%Scalable inference with Infinite Gaussian Process Hidden Markov Model for spatio-temporal data
%for Domain Adaptation
%on Function Spaces 
}



\icmlsetsymbol{equal}{*}

\begin{icmlauthorlist}
\icmlauthor{Sunrit Chakraborty}{equal,um}
\icmlauthor{Aritra Guha}{equal,duke}
\icmlauthor{Rayleigh Lei}{um}
\icmlauthor{XuanLong Nguyen}{um}
\end{icmlauthorlist}

\icmlaffiliation{um}{Department of Statistics, University of Michigan}
\icmlaffiliation{duke}{Department of Statistical Science, Duke University}

\icmlcorrespondingauthor{Aritra Guha}{aritra.guha@duke.edu }

\vskip 0.3in
]



\printAffiliationsAndNotice{\icmlEqualContribution} % otherwise use the standard text.

\begin{abstract}
Analysis of heterogeneous patterns in complex spatio-temporal data finds usage across various domains in applied science and engineering, 
including training autonomous vehicles to navigate in complex traffic scenarios. Motivated by applications arising in the transportation domain, in this paper we develop a model for learning heterogeneous and dynamic patterns of velocity field data. We draw from basic nonparameric Bayesian modeling elements such as hierarchical Dirichlet process and infinite hidden Markov model, while the smoothness of each homogeneous velocity field element is captured with a Gaussian process prior. Of particular focus is a scalable approximate inference method for the proposed model; this is achieved by employing sequential MAP estimates from the infinite HMM model and an efficient sequential GP posterior computation technique, which is shown to work effectively on simulated data sets. Finally, we demonstrate the effectiveness of our techniques to the NGSIM dataset of complex multi-vehicle interactions. 
%Our method may also be applicable to simulation and testing platforms for various traffic scenarios and other types of spatio-temporal data. 

%Rayleigh's suggestion:
%Analysis of heterogeneous patterns in complex spatio-temporal data finds usage across various scientific fields, such as biomedical research, climatic variation, satellite imaging, and spectrometry, as well as for training autonomous vehicles to navigate in complicated traffic scenarios. One model that can be used to analyze these patterns is the infinite hidden Markov model with each hidden state representing a different Gaussian process (GP) \citep{}. This nonparametric model can learn any number of spatial patterns while still preserving the temporal aspect of the data. However, it can be computationally challenging to fit such a model. In this paper, we devise a fast, approximate inference scheme for this model based on a sequential, greedy search method that is applicable for both univariate and multivariate GPs. We demonstrate that this technique works effectively on simulated data sets and the NGSIM dataset, which contains complex multi-vehicle interactions. 

\end{abstract}
\section{Introduction}

A common theme arising in many modern engineering applications is that there often is a large amount of data available via spatiotemporal dynamics generated in a potentially fast-paced and highly heterogeneous environment; yet there is a need to extract meaningful and interpretable patterns out of such complexities in a computationally efficient way. The learned patterns further enhance the user's understanding and improve subsequent decision-making. While there are many examples in a variety of domains, what motivates our present work the most is the analysis of traffic flow patterns out of high-volume and streaming measurements of vehicles passing through a busy thoroughfare.

\begin{figure}[t]
\centering
\includegraphics[width=0.49\textwidth]{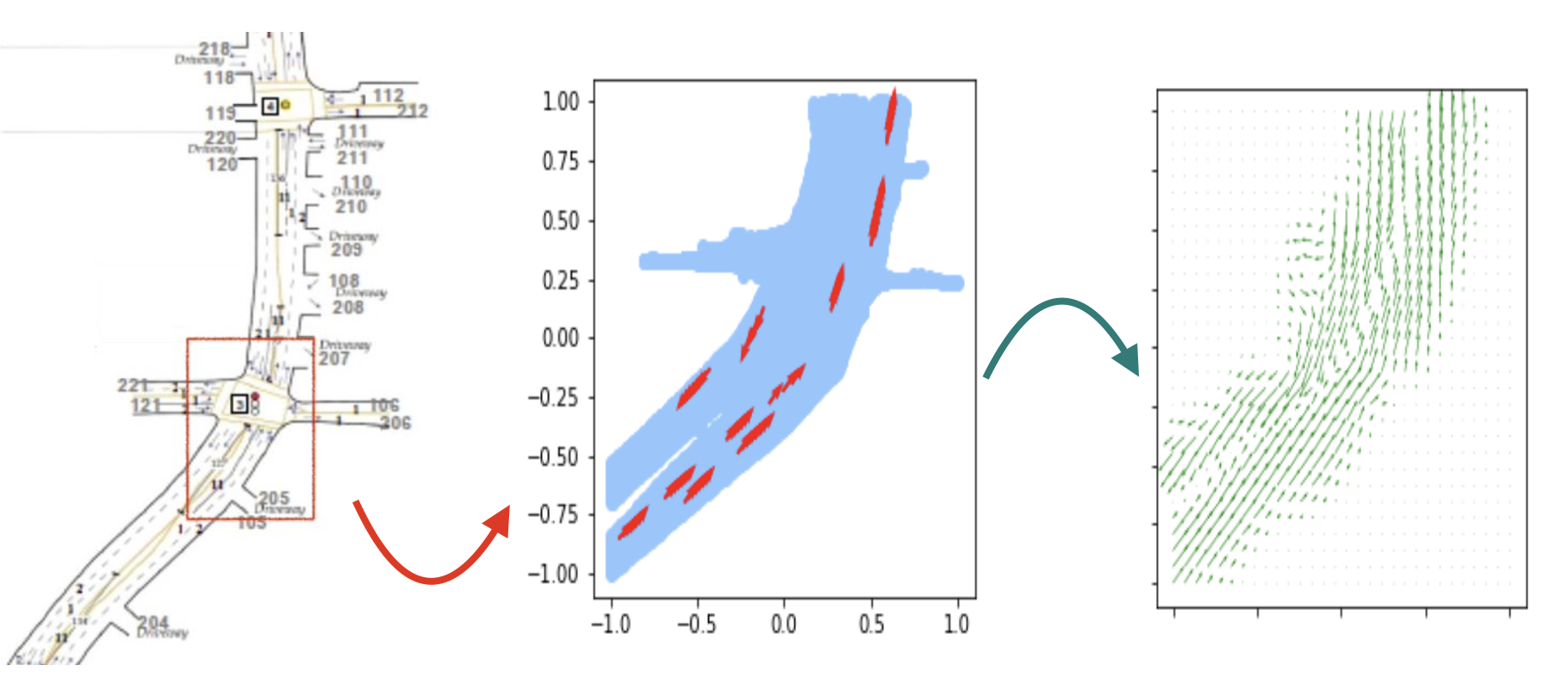}
\caption{Left figure shows a portion of a Los Angeles boulevard, middle figure shows a frame of traffic presence passing through an intersection; right figure shows the corresponding traffic flow pattern represented as a velocity field on $\mathbb{R}^2$ obtained by our method.}
\label{real}
\end{figure}

A newcomer to a large and busy city may be initially shocked upon observing a bewildering range of individual driving behaviors and of cars moving in varying speeds and directions, competing and challenging for an open lane at any given moment. Yet, underneath this seemingly intractable complexity, one may eventually find the calming ebbs and flows of movements regulated by traffic control systems and the rhythm of the day. Such patterns of traffic flows can be represented by a two-dimensional velocity field indexed on a two-dimensional plane (see  Fig.~\ref{real} for an illustration). The velocity field at a given time point records the expected velocity vector at different locations, if a car is present there at that moment. Unless there is an unusual disruption, one expects that the velocity vector varies smoothly, both in direction and magnitude, through the spatial domain. Thus, we adopt the viewpoint that a smooth vector field is a useful mathematical device to describe the current state of traffic flow at any given moment~\citep{guo2019modeling_dpgp,joseph2011bayesian_dpgp,Pedestrian_BNP}.

In this paper, motivated by the aforementioned application, and to provide a fast posterior inference algorithm for parameters and quantities of interest, we aim to create a probabilistic (Bayesian) model for learning smooth vector field patterns out of heterogeneous and dynamic time series data. Our starting point is to model a smooth velocity field after a multi-response Gaussian process defined on a spatial domain, an idea that was also explored in ~\citep{Kim:2011:Gaussian-Process}. To account for the temporal dynamics of spatial patterns, we employ a discrete-time hidden Markov chain that operates on the state space of smooth functions (representing the vector fields endowed by a Gaussian process prior). The vector fields are not observed directly; one only has access to frames of traffic passing through the road (see Fig.~\ref{real}). Moreover, to account for the highly heterogeneous environment of movements, we allow the number of hidden states to be unbounded. This is achieved by drawing from the powerful nonparametric Bayesian elements of infinite hidden Markov models (HMM) and hierarchical Dirichlet processes (HDP) \citep{beal2002infinite,Teh-etal-06}. 

%% In particular we use latent variable models to study the temporal patterns in the velocity field and use Gaussian process as the prior for capturing the spatial structure of the velocity field at a given time point.
%Today a large amount of data is generated by means of complex processes from heterogeneous mechanisms. Associated with each of these data-generation processes is a large number of hidden complexities. Some of the data such as facial recognition or flight pattern datasets may be extremely high-dimensional. Several others such as climate/weather patterns datasets may have hidden temporal aspects associated with them. A further other datasets such as text data, genetic or multiomics datasets may contain multiple heterogeneous patterns. Traffic driving encounter datasets are examples where all of the previously mentioned aspects simultaneously come to play. Any statistical model for them therefore needs to account for all these nuances in the data. Moreover, the model needs to be amenable for efficient inference, both statistically and computationally. Keeping these goals in mind, we propose in this paper a Bayesian nonparametric model to account for each of these complexities in datasets and provide methods for efficient inference. We apply our model to analyse driving behaviors available via the NGSIM dataset of complex multi-vehicle interactions.

In short, we propose an infinite hidden Markov model, in which the underlying Markov chain operates on the space of Gaussian process vector fields, and the measurement noise model also follows that of a Gaussian distribution.
Although the existing modeling elements are well-studied and have been explored in a wide range of applications, viz. Dirichlet processes for modeling heterogeneity~\citep{Ferguson-73,Antoniak-74,Ghosal-VDV-BNP-book}, hidden Markov models~\citep{Rabiner} and its infinite version~\citep{beal2002infinite,Teh-etal-06} for time series analysis, and Gaussian processes for spatial data~\citep{Cressie-93,Kim:2011:Gaussian-Process}, combining all such elements into a single nonparametric Bayesian modeling framework and applying it to high-dimensional velocity field data seems new and quite exciting for the application we have in our hand. 

%Dirichlet process~\citep{Ferguson-73,Antoniak-74,Blackwell-MacQueen-73,Ghosal-VDV-BNP-book} and its variants have long been used for modeling heterogeneous patterns~\citep{Leroux-92} in data and have also been adopted as building blocks for more sophisticated hierarchical modeling~\citep{Teh-etal-06,Rodriguez-etal-08}, especially because of the ease of implementation via efficient Markov Chain Monte Carlo (MCMC)~\citep{Neal-00,Escobar-West-95, MacEachern-Muller-98, Jain-Neal-07,Arash-HDP-slice} sampling procedures. The versatility of Dirichlet/Hierarchical Dirichlet Process (HDP) precludes the need to pre-specify the number of heterogeneous mechanisms due to its nonparametric structure, thereby making any inference on the components completely "automatic". On the other hand, Hidden Markov models (HMM)~\citep{Rabiner} are useful tools to model temporal patterns in data and have been used extensively in fields of thermodynamics, economics, finance, signal processing, information theory, pattern recognition (e.g. speech, handwriting, and gesture recognition~\citep{HMM-video} and musical score following~\citep{Hmm-music}), and bioinformatics~\citep{genetics-HMM}. Fast inference for HMMs are carried out by the Baum-Welch~\citep{Rabiner-2013-BaumWelch} and the Viterbi algorithms~\citep{Viterbi}. As the third component in our model, we make use of Gaussian Processes (GPs)~\citep{Rasmussen-GP} to model velocity fields across spatial domains~\citep{Kim:2011:Gaussian-Process}. 

Due to the complexity of the proposed model, a particular focus of this work is on the development of a scalable approximate inference method to overcome the shortcoming of existing computational approaches. The standard techniques for Bayesian inference include MCMC~\citep{Gelfand-MCMC,Fox-etal-09} or variational inference (VI)~\citep{Blei-et-al,SVI-HMM}. Due to the large number of latent variables in combination with complex modeling structures, MCMC algorithms tend to be inefficient. On the other hand, VI algorithms (cf.~\citep{jordan1999introduction,Blei-VI-DPMM,hoffman2013stochastic,mandt2017stochastic}) are known to have difficulty producing statistically accurate posterior distributions, especially for finite samples. Our computational innovations include employing sequential MAP estimates from the infinite HMM model and efficient sequential GP posterior computation techniques. The latter techniques are crucial in overcoming very large covariance matrix, which is a consequence of the GP observed at a large number of spatial locations. They include using a block matrix inversion matrix using Schur's complement. 
% and a spectral decomposition technique for handling voluminous isotropic data that are observed at the same spatial locations at every time point. 
As we demonstrate in Table~\ref{table:1} and~\ref{table:3}, these innovations allow us to analyze 10,000 total observations in around two minutes.

In summary, our contributions in this work are three-fold. Firstly, we study an infinite hidden Markov model on state space of multi-dimensional vector fields supported by a smooth Gaussian process prior.  Secondly, we provide explicit computations via MAP estimates and devise a fast inference algorithm for the proposed model. Thirdly, the application to understanding of traffic encounters is a novel utilization of the model and the algorithm. 
%While nonparametric algorithms involving Dirichlet Process and Gaussian processes have been used before to analyze motion patterns for trajectory prediction of vehicle and pedestrian movements~\citep{guo2019modeling_dpgp,joseph2011bayesian_dpgp,Pedestrian_BNP}, 

Other related work include~\citep{Fox-etal-11}, in which an infinite HMM combined with HDP has been used successfully to model speaker diarization behavior~\citep{Fox-etal-11}. By contrast, our work appeals to an infinite HMM for the high-dimensional velocity field hidden state space.
There have also been prior work that combines both DP and GP modeling elements ~\citep{guo2019modeling_dpgp,joseph2011bayesian_dpgp,Pedestrian_BNP}. The temporal modeling of the patterns in our work brings forward a novel aspect to the application perspective, which is potentially useful in improving autonomous vehicles based on interpretable learned patterns. Moreover, previous implementations of the DP-GP algorithms~\citep{guo2019modeling_dpgp} are incapable of dealing with presence of large number of agents in each temporal epoch. As demonstrated in Section~\ref{section:experiments}, our computational techniques help to overcome this shortcoming effectively.

% Our contributions include modeling functional time series data using the infinite HMM model (which is equivalent to the infinite HMM model) with each latent state corresponding to a function and we use a Gaussian process prior to estimate these latent functions. We devise a novel fast inference algorithm for this model where we sequentially update the parameters by computing the posterior mode conditioned on the history, which enables us to get approximately optimal solution but at a much faster time than the usual MCMC based inference. We demonstrate our algorithm on a couple of simulated datasets which shows that our model and algorithm can successfully recover the true underlying spatio temporal patterns if the data is not misspecified, it can not only estimate the true number of latent clusters but also captures the functions (clusters) as well as the temporal dynamics of the HMM. We also apply our method to the NGSIM dataset and compare our result to the DP-GP model and demonstrate that it captures the spatio temporal patterns and is much faster. To our knowledge we do not know of any other modeling framework which estimates the number of latent clusters, estimates the clusters (underlying functions) and also estimates the temporal dynamics of such complex functional data using non parametric Bayesian methods.

The remainder of the paper is organized as follows. In section~\ref{section:review} we briefly review existing ideas necessary for the remainder of the paper, section~\ref{section:model_and_inference} describes our model. Section~\ref{section:efficient_computation} harps on the inference algorithm while section~\ref{section:experiments} demonstrates experimental results on simulated datasets and NGSIM traffic data.

% Meanwhile, we assume that we have $K^{T}$ GPs. We denote these functionals as $\widetilde{\phi}_k$, $k \in {1, 2, \ldots, K^{T}}$.
% We let $\phi_t$ denote the hyperparameters for the GP associated with the hidden state

\section{Preliminaries}
\label{section:review}
In this section, we briefly describe several key Bayesian nonparametric modeling elements for clustering data based on latent topics with unknown number of clusters and latent temporal dynamics. We also describe Gaussian processes and multivariate response Gaussian processes, which we use as the prior on the space to smooth velocity fields.

\subsection{Infinite HMM}
The infinite hidden Markov model was first proposed in \citep{beal2002infinite} and subsequently shown to be an instance of the general Hierarchical Dirichlet process model of \citep{Teh-etal-06}. We describe the infinite HMM setup as follows.

Assume that the behavioral outcome observed at each time-point is a noisy version of a specific underlying pattern among infinitely many such possible patterns. Let $\phi_1, \phi_2, \dots \overset{iid}{\sim} H$  be used to denote the underlying patterns,  with $\phi_k$ used to assemble the pattern associated with the $k^{th}$ component. On the other hand, at each time point $t$, we have a random variable, $s_t \in \mathbb{N}$, which denotes which pattern is active at time $t$. The key assumption underlying the (hidden) Markov model structure is that the active pattern at time $t$ conditioned on the active pattern at $t-1$ is independent of prior history of active patterns. 

Specific to the infinite HMM setup, the choice of pattern at each step $t$ affects the hidden pattern active at $t+1$ via an oracle value $o_t$. If $o_t=0$, the choice of active pattern depends on the historical counts of respective pattern types, whereas if $o_t=1$, an oracle is invoked.
%(whose associated counts we denote by $m_j^{(t)}=\sum_{u=1}^{t-1} \mathbbm{1}(s_{u+1} =j, o_{u} = 1)$)  
Before we mathematically define the model, we introduce some notations for count variables that will be useful:
\begin{align}
    N^{(t)}_j &= \sum_{u=1}^t \mathbbm{1}(s_u = j),\;
    n^{(t)}_{ij} = \sum_{u=1}^{t-1} \mathbbm{1}(s_u=i, s_{u+1}=j), \label{eq:n} \\
    m^{(t)}_j &= \sum_{u=1}^{t-1} \mathbbm{1}(s_{u+1} = j, o_u = 1). \label{eq:m}
\end{align}
They respectively represent the number of times a state has been visited, the number of transitions from one state to another, and the number of times a state has been visited while invoking the oracle until time $t$. If $s_1,\dots, s_t$ are the states for time points up to $t$ and $\tilde{K}^{(t)}$ are the number of distinct states explored,
% Let $s_1,\dots, s_t$ be the states for time points up to $t$, $K^{t}$ be the number of distinct states explored and we obtain the counts $n_{ij}^{t}$ counting number of transitions from $i$ to $j$. Then, 
the infinite HMM model (with parameters $\alpha$ and $\gamma$) is completely described by the process of sampling $s_{t+1}$. This is done in the following manner.
\begin{align}
    \mathbb{P}(s_{t+1} = j, o_t=0 \mid s_t=i, n,\alpha) &= \frac{n_{ij}^{(t)}}{\sum_{j'=1}^{\tilde{K}^{(t)}} n_{ij'}^{(t)} + \alpha} \nonumber \\
    \mathbb{P}(o_t= 1\mid s_t=i, n,\alpha) &= \frac{\alpha}{\sum_{j'=1}^{\tilde{K}^{(t)}} n_{ij'}^{(t)} + \alpha} \label{eq:step1}
    \end{align}
Moreover, given that we have defaulted to an oracle ($o_t=1$), the transition satisfies
\begin{align}
    s_{t+1}\mid &s_t=i, m, \gamma, o_t=1 \sim \nonumber\\ &\sum_{j=1}^{\tilde{K}^{(t)}} \frac{m_j^{(t)}}{\sum_{j'} m_j^{(t)} + \gamma} \delta_j + \frac{\gamma}{\sum_{j'} m_j^{(t)} + \gamma} \delta_{\tilde{K}^{(t)}+1} \label{eq:step2}
    \end{align}
When $K+1$ is chosen, the system explores a new state $K+1$ and gets $\phi_{K+1}$ previously unused. This two layer structure achieves the same objective as the HDP (described in the \textbf{Appendix}). Now to complete the HMM structure, when $s_t = k$, we assume that the observation emission follows $x_t \sim F(\cdot| \phi_k)$. This completes the description of the Infinite HMM model. The model is illustrated in Figure \ref{iHMM}.

\begin{figure}[h]
\centering
\includegraphics[width=0.45\textwidth]{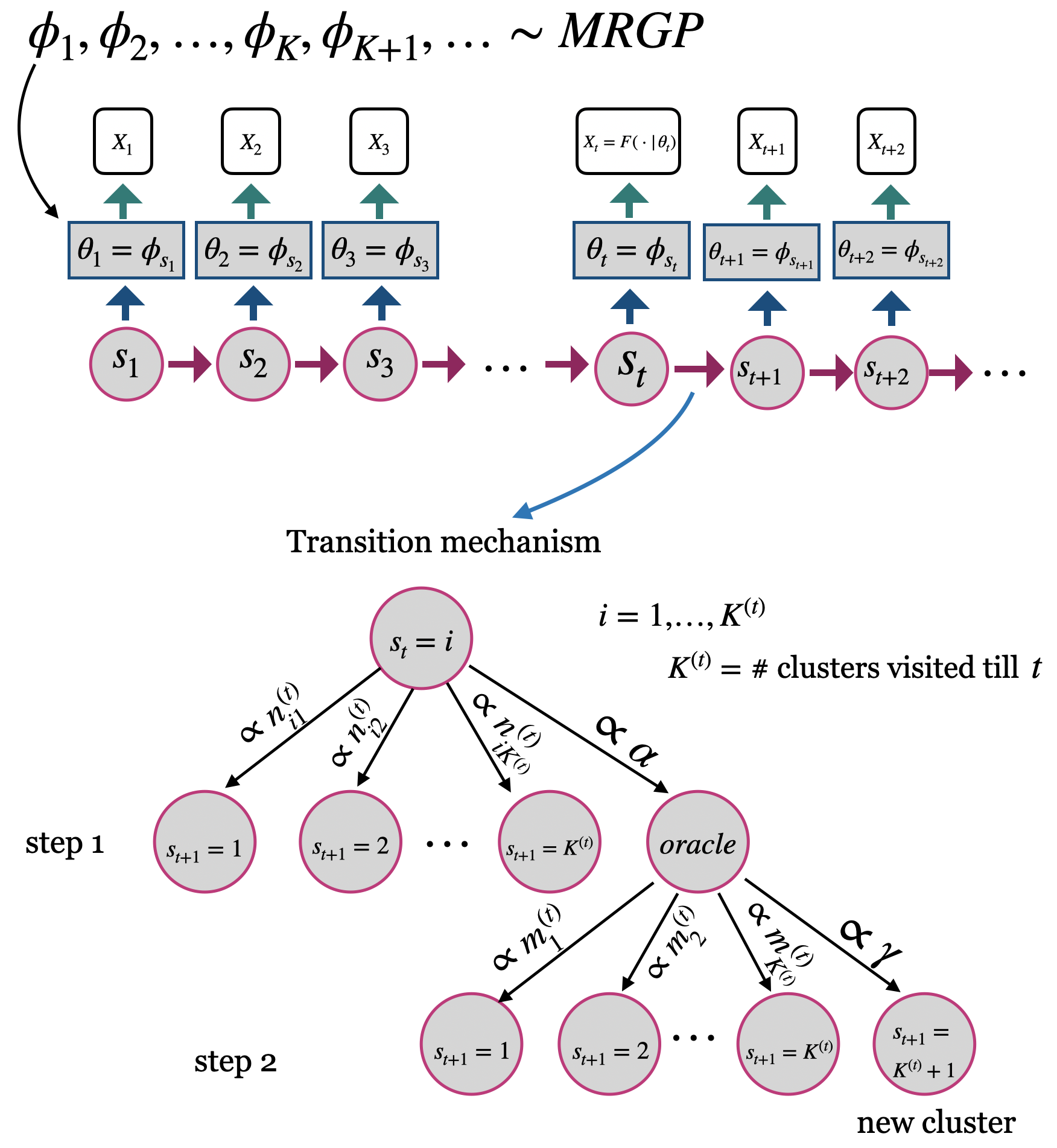}
\caption{Graphical illustration of the infinite HMM model.}
\label{iHMM}
\end{figure}

\begin{figure*}[h]
\centering
\includegraphics[width=0.95\textwidth]{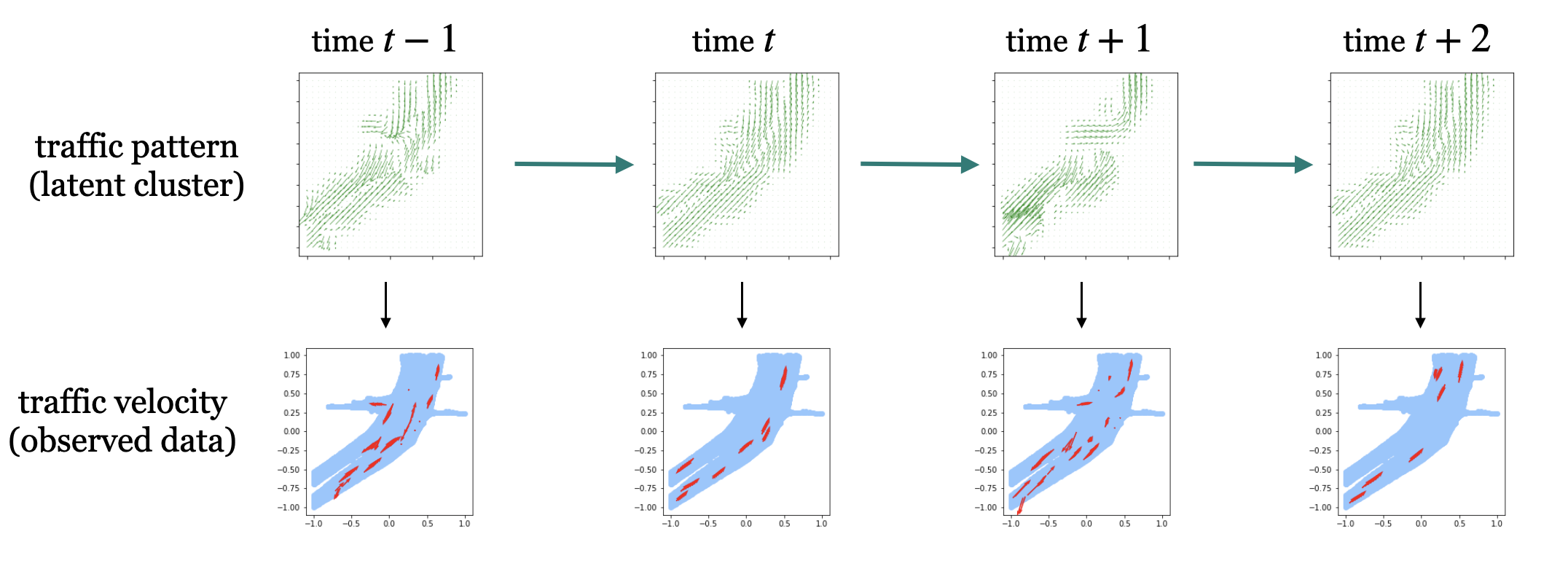}
\caption{A simple example of our model: at each time $t$, a latent cluster is chosen and based on that traffic pattern, we observe some real data --- velocity vectors at some spatial points. Note in this example the clusters at time $t$ and $t+2$ are same --- the frames at these time points correspond to the same traffic pattern.}
\label{temporal}
\end{figure*}

\subsection{Gaussian process}
Gaussian processes (GP) provide a mechanism to model (smooth) functions on arbitrary index spaces. A stochastic process $\{X(t):X(t)\in \mathbb{R},\ t\geq 0\}$ is called a GP with mean $m(\cdot)$ and covariance kernel $K(\ast,\ast)$  if for any finite $T:=\{t_1,\dots,t_k\} \subset [0,\infty)$,
\begin{eqnarray}
(X(t_1),\dots,X(t_k)) \sim \mathcal{N}\left(m\big|_T,K\big|_{T\times T}\right),
\label{distr:mvn}
\end{eqnarray}
%jointly, 
where $m\big|_T=(m(t_1),\dots,m(t_k))$ and $K\big|_{T \times T}(i,j)=K(t_i,t_j), \text{ } 1 \leq i,j\leq k$.

\subsection{Multi-response Gaussian process}
Before we introduce the multi-response Gaussian process (MRGP), we need to define the matrix normal distribution. This distribution will allow us to assign probability to the stochastic process. 

A random matrix $Z_{k\times d}$ is said to follow a matrix normal distribution with parameters $M_{k \times d}, U_{k \times k}$ and $V_{d \times d}$, i.e., $Z \sim \mathcal{MN}_{k\times d}(M, U, V)$, if
\begin{eqnarray}
\text{vec}(Z)\sim {\mathcal {N}}_{kd}(\mathrm {vec} (M),V \otimes U ).
\label{distr:mnr}
\end{eqnarray}
The Kronecker product is denoted by $\otimes$  and ${ \mathrm {vec} (M )}$ signifies the vectorization of $M$. We now define the MRGP. 

Let $f:\mathbb{R}^p \to \mathbb{R}^d$ and we write for $z\in \mathbb{R}^p$, $f(z) = (f_1(z), \dots, f_d(z))^T \in \mathbb{R}^d$. Given a kernel $K:\mathbb{R}^p \times \mathbb{R}^p \to \mathbb{R}^+$ and a mean function $\mu:\mathbb{R}^p\to\mathbb{R}^d$, we write $f \sim MRGP(\mu, K, \rho)$ if for any finite $n$ and any $z_1,\dots,z_n \in\mathbb{R}^p$, we posit the following matrix normal distribution 
\begin{equation}
    f(z_{1:n}) = \begin{pmatrix}
    f_1(z_1)  & \dots & f_d(z_1) \\
    f_1(z_2)  & \dots & f_d(z_2) \\
    \vdots  & \ddots & \vdots \\
    f_1(z_n)  & \dots & f_d(z_n)
    \end{pmatrix} \sim MN_{nd}\left(M, \Sigma, \Omega\right).
    \label{model:mrgp}
\end{equation}
Here, $z_{1:n}=(z_1,\dots,z_n)$, $M\in\mathbb{R}^{n\times d}$ with $M_{ij} = \mu_j(z_i)$, 
\begin{align}
    \Sigma &= \begin{pmatrix}
    K(z_1,z_1) & K(z_1,z_2) & \dots & K(z_1,z_n) \\
    K(z_2,z_1) & K(z_2,z_2) & \dots & K(z_2,z_n) \\
    \vdots & \vdots & \ddots & \vdots \\
    K(z_n,z_1) & K(z_n,z_2) & \dots & K(z_n,z_n)
    \end{pmatrix}\\
    \Omega &= \begin{pmatrix}
    1 & \rho & \dots & \rho \\
    \rho & 1 & \dots & \rho \\
    \vdots & \vdots & \ddots & \vdots \\
    \rho & \rho & \dots & 1
    \end{pmatrix}.
\end{align}
In other words, $\Sigma$ captures the covariance across the rows and $\Omega$ across the columns. In our case, we fix $\Omega$ as the  equicorrelation($\rho$) matrix of size $d\times d$ and $\Sigma$ is a $n\times n$ matrix formed using the kernel $K$ as $\sigma_{ij} = K(z_i,z_j)$.

We choose $\mu \equiv 0$ and we use the Radial Basis Function (RBF) kernel $K(x_1, x_2) = \sigma_0^2 \exp\left\{-\frac{\norm{x_1-x_2}^2}{2\ell_0^2}\right\}$ where $\sigma_0^2$ is the kernel variance and $\ell_0$ is the kernel lengthscale.

\section{Data model}
\label{section:model_and_inference}
We assume that the data is spatio-temporal in nature. More specifically, given an underlying spatial domain $\mathcal{B} \subset \mathbb{R}^p$. Let us denote $\mathcal{F}=\{f:B \to \mathbb{R}^d\}$ as a space of functions with domain in $B$ and range in $\mathbb{R}^d$. We are given a stochastic process, $\{X_t:X_t\in \mathcal{F},\ t\geq 0\}$, which we wish to model. In other words, at each discrete time-point $t=1,2,\dots,T$, we have a system that outputs a function. Moreover, at each time point $t$, we only observe the outputs $X_t(z_1),\ldots, X_t(z_N)$, for some $z_1,\ldots, z_N \in \mathcal{B}$.

The \textbf{key assumption} underlying our model is that there exist an \textbf{unknown number of true patterns (or functions) $\phi^0_1,\phi^0_2, \dots \in \mathcal{F}$} which give rise to the observed patterns as follows. Suppose at time point $t$, pattern $s_t$ is active, then the observations at time point $t$ are modeled as:
\begin{eqnarray}
\label{eq:noise}
X_t(z_n) &\sim \mathcal{N}(\phi^0_{s_t}(z_n),  \sigma^2\mathbbm{I}) \qquad n = 1, 2, \ldots, N.
\end{eqnarray}

We next discuss how to model the random selection of patterns at each time $t$ by drawing from our intuition about modeling velocity flow patterns relevant to traffic movements. Traffic flow patterns at a time point are directly influenced by the patterns of traffic lights. How other patterns affect flow patterns might depend on the time of day, which in turn affect how the flow patterns behave locally in time. While it is expected that flow patterns at time points close to each other would be strongly dependent, it is reasonable to model the flow patterns as independent whenever they are separated by a large time interval. In that regard, Markov chains form the simplest objects to model changes in behavior locally across time. For this paper we will focus on 1-step Markov Chain via a hidden Markov model for choosing states. The movement of the Markov chain is guided by transition probabilities between different states. Since we want to be flexible about the number of states, we allow for an infinite number of latent states, each having an infinite length transition probability vector for moving to the next state. Infinite HMMs therefore provide an appropriate setup to model such transitions.

Moreover, we want to be flexible about the nature of the velocity flow. The basic assumption underlying a velocity flow is that each location in a region ($\mathcal{B}$ in this case) is associated with a velocity. The collection of all the velocities across all such locations is a velocity field. GPs are flexible objects for modeling arbitrary multivariate functions on spatial domains. We therefore assume that each hidden velocity field pattern, labelled as $\phi_1,\phi_2,\ldots$, (different from the true underlying patterns $\phi^0_1,\phi^0_2, \dots$ ) is modelled as realisations from a MRGP with RBF kernel $K(\cdot,\cdot)$ in a suitable domain. 

\paragraph{Model:}
The complete model is outlined as follows.
\begin{equation}
    \begin{aligned}
        \phi_1, \phi_2, \dots  &\sim MRGP(0,K,\rho) \\
        s_1, \dots, s_T &\sim \text{infinite HMM}(\alpha, \gamma) \\
        X_t(z_n) \mid \{\phi_k\} &\sim \mathcal{N}(\phi_{s_t}(z_n), \sigma^2\mathbbm{I}) & t = 1, 2, \ldots, T;\\
        & & n = 1, 2, \ldots, N.    \label{eq:model_inf_hmm_gp}
    \end{aligned}
\end{equation}
Note that in the model, we assume that $z_1,\ldots,z_N$ are common to all time-points $t$, but our model can be easily extended to the case of observing velocity flows in different locations across different time points. The analysis remains similar to the one performed below. We focus on this scenario to avoid over-burdening our notations.

The usefulness of the above model is multi-fold. First, it helps to extract each pattern of traffic movement corresponding to a given time point. Moreover, it provides us the ability to infer about the transition patterns. In the context of autonomous vehicles, while this is extremely useful to guide the vehicle about the current scenario of neighboring traffic, it also provides an understanding about what behavior to expect from neighboring vehicles at the next instant.

\section{Fast sequential posterior computation for Gaussian process}
\label{section:efficient_computation}
% In this section, we illustrate how we can efficiently sequentially update the posteriors of the univariate and multivariate response GPs. We also describe a faster sequential procedure for computing the posterior predictive distributions of multivariate response GPs.

% consider the univariate and multivariate response GP. We demonstrate how we can efficiently sequentially update the posteriors of the GPs for the former whereas for the latter, we illustrate a faster sequential procedure for computing the posterior and the posterior predictive distributions.
% the first that of a single response GP considering isotropic points where  and the second that of a multi-response GP where we also show a faster sequential procedure for computing the posterior and the posterior predictive distributions.\subsection{Inference Algorithm}
The full posterior with the above model is a complex object. While MCMC updates can be extremely slow due to invoking of forward-backward algorithm (especially with high-dimensional calculations with GPs), approximate techniques such as variational inference can often lead to inaccurate estimates. We therefore focus on maximum a posteriori (MAP) estimates for inference. 

Our particular inference scheme involves sequentially estimating the state variables, $s_t$, and oracle indicator variables, $o_t$ for $t = 1, 2, \ldots, T$ and the latent, spatial functions $\phi_k$ for $k = 1, 2, \ldots, \tilde{K}^{(t)}$. The steps for doing so via a one-pass MAP estimator are given in Algorithm~\ref{alg 1}. Before we elaborate on the computation of the different steps in Algorithm~\ref{alg 1}, the following notation will be helpful to describe this algorithm:

 Let $\tilde{K}^{(t)}$ be the number of observed patterns until $t$. Also, let $\mathcal{H}_t=\biggr\{\phi_{1:\tilde{K}^{(t)}},
    o_{1:t},s_{1:t},\\ \hspace{4 em}\{m_j^{(t')}\}_{k; t'=1:t},\{n_{ij}^{(t')}\}_{k,k'; t'=1:t}, \{N_j^{(t')}\}_{j; t'=1:t}\biggr\}$. Here, $m_j^{(t')},n_{ij}^{(t')},N_j^{(t')}$ are as defined in Eq.~\eqref{eq:m}.

% The MAP update estimates $s_{t + 1}$ and $o_{t + 1}$ at each time $t$ based on the data and the count and previously estimated variables until that time. These estimates are straightforward to calculate. 

% The steps for the one-pass MAP estimator are given in Algorithm~\ref{alg 1}. Next we elaborate on the computation of the different steps in Algorithm~\ref{alg 1}. 
\subsection{ Estimating state variable}

Let $z^{t+1}_{1:n^{(t+1)}}$ denotes the locations of observation at time $t+1$. By Bayes' rule, the posterior distribution of $s_{t + 1}$ is 
\begin{align}
    \Prob(s_{t+1}&=j \mid \mathcal{H}_t, X_1, X_2, \ldots, X_{t+1}) \propto \label{eq:s_post_summary} \\
    &\Prob(s_{t+1}=j|\mathcal{H}_t) \Prob(X_{t+1}(z^{t+1}_{1:n^{(t+1)}}) \mid s_{t+1}=j, \mathcal{H}_t) \nonumber
\end{align}
for $j = 1, 2, \ldots, \tilde{K}^{(t)} + 1$.  
% We can use the previous count variables to help compute the former probability and conjugacy to compute the latter probability. 

We can use the transition probabilities for infinite HMM given in Eq.~\eqref{eq:step1} and Eq.~\eqref{eq:step2} to get that
\begin{eqnarray}
\label{eq:prior}
    & &\Prob(s_{t+1}=j \mid s_t = i, \mathcal{H}_t)  \\
    &=&
    \begin{cases}
        \frac{n^{(t)}_{ij}}{\sum_{j'=1}^{\tilde{K}^{(t)}} n^{(t)}_{ij'} + \alpha} + \frac{\alpha m^{(t)}_j}{(\sum_{j'=1}^{\tilde{K}^{(t)}} n^{(t)}_{ij'} + \alpha)(\sum_{j'=1}^{\tilde{K}^{(t)}} m^{(t)}_{j'} + \gamma)} , \nonumber \\
        \hspace{12 em} \text{ if }1 \leq j \leq \tilde{K}^{(t)} \nonumber \\
        \frac{\alpha \gamma}{(\sum_{j'=1}^{\tilde{K}^{(t)}}n^{(t)}_{ij'} + \alpha)(\sum_{j'=1}^{\tilde{K}^{(t)}} m^{(t)}_{j'} + \gamma)}, \text{ if }j = \tilde{K}^{(t)}+1. \nonumber 
    \end{cases}
\end{eqnarray}
The first line refers to some previous state $j$ being chosen at time $t+1$ and the first term is when it is chosen directly while the second term is for when it is chosen through the oracle. The second line refers to a new state being chosen $s_{t+1}=\tilde{K}^{(t)}+1$, which is only possible through the oracle. Eq.~\eqref{eq:prior} defines a prior for $s_{t+1}$ given all the required terms.

The computation of the second term in the RHS of Eq.~\eqref{eq:s_post_summary} is computed using Prop.~\ref{prop:prediction_X} and is provided in the appendix.

\subsection{Estimating oracle variable}

The posterior distribution of $o_{t + 1}$ is calculated using Bayes' rule as follows.

For $j = 1, 2, \ldots, K^{(t)} + 1$, $e \in \{0, 1\}$, by Lemma A.1 in the appendix,
\begin{align}
    \Prob(o_{t+1} =  e&\mid X_{t+1}(z^{t+1}_{1:n^{(t+1)}}), s_{t+1}= j, \mathcal{H}_t) \propto  \label{eq:o_post_summary} \\
    &\Prob(s_{t+1}=j\mid o_{t+1}=e, \mathcal{H}_t)\Prob(o_{t+1} = e \mid \mathcal{H}_t), 
   \nonumber
\end{align}

% These probabilities are also easy to calculate using the previous count variables. 

% Now we discuss the posterior distribution of $o_{t+1}$ given all other variables till time $t$, $s_{t+1}$ (we shall use the estimated value for $s_{t+1}$ for this) and the observation $x_{t+1}$. 

Recall that $o_t$ is a binary variable, which is 1 if $s_{t+1}$ was generated through the oracle and $0$ if $s_{t+1}$ was generated directly. We first describe the former case. For the \textbf{first term} in the RHS of Eq.~\eqref{eq:o_post_summary}, Eq.~\eqref{eq:step2} tells us that
\begin{eqnarray}
\label{eq:os_1}
& & \Prob(s_{t+1} = j\mid o_{t+1} = 1, \mathcal{H}_t) \nonumber\\ &=&
\begin{cases}
m^{(t)}_{j}/(\sum_{j'} m^{(t)}_{ij'} + \gamma),   \    j\in \{1,\dots, \tilde{K}^{(t)}\} \\
\gamma/(\sum_{j'} m^{(t)}_{ij'} + \gamma),   \  j=\tilde{K}^{(t)}+1.
\end{cases}
\end{eqnarray} 
In other words, it is the probability that the oracle was invoked to generate the next hidden state, $j$. Using Eq.~\eqref{eq:step1}, we get that the \textbf{second term} in the RHS of Eq.~\eqref{eq:o_post_summary} is
\begin{align}
\label{eq:o_1}
    \Prob(o_{t+1}=1\mid s_t = i, \mathcal{H}_t) = \alpha/(\sum_j n^{(t)}_{ij} + \alpha).
\end{align} 
This is the probability that the oracle is invoked at time $t+1$. While $s_t = i$ is contained in $\mathcal{H}_t$, we explicitly write it out to make clear that this probability depends on the number of transitions from state $i$.  

We can then make similar calculations for $o_t = 0$. For the first term, we have that 
\begin{eqnarray}
\label{eq:os_2}
    & &\Prob(s_{t+1} =  j\mid o_{t+1} = 0, \mathcal{H}_t) \nonumber \\
    &=&\begin{cases}
    n^{(t)}_{ij}/\sum_j n^{(t)}_{ij} , \ \ j\in \{1,\dots, \tilde{K}^{(t)}\} \\
    0 , \ \ \ \  \qquad j = \tilde{K}^{(t)}+1 .
    \end{cases}
\end{eqnarray}
The second term is $1 - \Prob(o_{t+1}=1\mid s_t = i, \mathcal{H}_t)$.

\begin{algorithm}[h]
\SetAlgoLined
\textbf{Input:} Data $\{X_t\}_{t=1}^\infty$ fed sequentially. Hyperparameters $\sigma^2$,kernel $K$, locations $\{z^{t}_{1:n^{(t)}}\}_t, \rho,\alpha,\gamma$.\\
 \textbf{Initialization:} \\
 For $t=1$, set $s_t = 1, o_t=1, n^{(1)} = [0], m^{(1)} = [1]$\\
 Update $\hat{\phi}_1$ using Proposition~\ref{prop:prediction_X}\\
 \textbf{Steps:} For each time $t = 1, 2, \dots, T$
 \begin{enumerate}
     \item Set $\hat{s}_{t+1} \xleftarrow{} \argmax_j \Prob(s_{t+1} = j \mid \mathcal{H}_{t}, x_{t+1})$ using \eqref{eq:s_post_summary}, \eqref{eq:prior} and Proposition~\ref{prop:prediction_X}.
    %  where we use $\hat{s}_u$ for all $s_u$, $u \leq t$
     \item Set $\hat{o}_{t+1} \xleftarrow{} \argmax_e \Prob(o_{t+1}=e \mid x_{t+1}, s_{t+1}, \mathcal{H}_t)$, using Eqs.\eqref{eq:o_post_summary}, \eqref{eq:os_1},\eqref{eq:o_1} and \eqref{eq:os_2}. 
    %  again replacing $s_u$ by $\hat{s}_u$ for all $u\leq t+1$.
     \item Update $n^{(t)}_{\widehat{s}_{t}, \widehat{s}_{t + 1}}, m^{(t)}_{\widehat{s}_{t + 1}}$.
     \item Update estimates for $\phi_{s_{t+1}}$ using the GP posterior discussed in (U.1) at the end of Section \ref{ssection:patterns}.
 \end{enumerate}
 \textbf{Output:} $\widehat{s}_t$ and $\widehat{o}_t$ for $t = 1, 2, \ldots, T$ and $\hat{\phi_k}$ for $k = 1, 2, \ldots, K_0$
%  and transition count matrix $n$\\
 \caption{Sequential maximum a posteriori estimation for Infinite HMM-GP}
 \label{alg 1}
\end{algorithm}

These estimates, $\widehat{s}_{t + 1}$ and $\widehat{o}_{t + 1}$, are then used to update the $\phi_{\widehat{s}_{t + 1}}$ and previous count variables. We describe how to efficiently sequentially update $\phi_k$ in the next section. 

\subsection{Estimating underlying patterns }
\label{ssection:patterns}
We now discuss the estimation of underlying patterns $\phi_1,\phi_2,\ldots$ at time-step $t+1$. 

\textbf{Notations:}
\begin{enumerate}
    \item[(P.1)] Let $T_j^t=\{t^j_1, t^j_2, \ldots, t^j_{j_t}\}$, $j_t\in \mathbbm{N}$, indicate all the times during for which $\widehat{s}_{t'} = j$ for $t' \leq t$.
    \item[(P.2)] Let $x_t=\text{vec}(X_t(z_{1:n^{(t)}}))$ be the $n^{(t)}d$ dimensional vector of observations at time $t$.
    \item[(P.3)] Moreover, let $\text{vec}\left((x^{t^j_k})\mid_{k=1:j_t}\right)$ be the $Nd$ ($N$ is the total number of locations among elements of $T_j^t$) dimensional vector obtained by stacking $x_{t^j_{k}}$, $k=1:j_t$, on top of one another.  
    \item[(P.4)] Let $Z_{t}^j$ denote the collections of all the locations for observations across time-points in $T_j^t$. Then, let $K(Z_{t}^j,Z_{t}^j)$ denote the matrix $(K(z_i,z_j))_{z_i,z_j \in Z_{t}^j}$. Similarly, define $K(z,Z_{t}^j)$,$K(Z_{t}^j,z)$ and $K(z,z)$ for any $z \in \mathbb{R}^p$.
\end{enumerate}

By the assumption of MRGP we have that,
\begin{eqnarray}
    & &\text{vec}\left((x^{t^j_k})\mid_{k=1:j_t}\right)\sim \\ & & 
    \hspace{0.5 em} \mathcal{N}_{Nd}\left({0}, K(Z_{t}^j,Z_{t}^j)\otimes \Omega(\rho) + I_N \otimes (\sigma^2 I_d)\right) \nonumber.
    \label{eq:multi_response_gp_l}
\end{eqnarray}
Based on this, we can use the conditional normal distribution to update $\phi_j$ given $\hat{s}_{t+1}=j$ as follows.
 
\begin{prop}
\label{prop:prediction_X}
Given notations in (P.1)-(P.4) and $\hat{s}_{t^j_1}, \ldots, \hat{s}_{t^j_{j_t}}=j$, we have that for any $z \in \mathbb{R}^p$.
\begin{eqnarray}
    & &\text{vec}(\phi_j(z)) | x_{t^j_1}, \ldots, x_{t^j_{j_t}},\hat{s}_{t^j_1}, \ldots, \hat{s}_{t^j_{j_t}}=j\nonumber \\ &\sim &\mathcal{N}_{Nd}\left(\mu^*, \Sigma^*\right)
\end{eqnarray}
where
\begin{align}
\label{eq:Posterior_mean_var}
    \Lambda &= \left(K(Z_{t}^j,Z_{t}^j)\otimes \Omega(\rho) + I_N\otimes \sigma^2 I_d\right)^{-1}, \\
    \mu^* &= \left(K(z,Z_{t}^j)\otimes \Omega(\rho)\right)\Lambda \text{vec}\left((x^{t^j_k})\mid_{k=1:j_t}\right), \nonumber\\
    \Sigma^* &= K(z,z)\otimes\Omega(\rho) - \nonumber\\ &\quad\left(K(z,Z_{t}^j)\otimes \Omega(\rho)\right)\Lambda\left(K(Z_{t}^j,z)\otimes \Omega(\rho)\right). \nonumber
\label{sigma}
\end{align}
The posterior predictive distribution of $X_{t+1}(z^{t+1}_{1:n^{(t+1)}})$ is then simply
\begin{eqnarray}
    & & X_{t+1}(z^{t+1}_{1:n^{(t+1)}})|x_{t^j_1}, \ldots, x_{t^j_{j_t}} \nonumber \\ &\sim& \mathcal{N}_{n^{(t+1)}d}\left(\mu^*, \Sigma^* + I_{n^{(t+1)}}\otimes \sigma^2 I_d\right)
\end{eqnarray}
with $z=z^{t+1}_{1:n^{(t+1)}}$ in Eq.~\eqref{eq:Posterior_mean_var}.
\end{prop}

Note that to compute $\Lambda$, which is central to calculating $\mu^*$ and $\Sigma^*$, we need to estimate $\rho$ and invert the matrix $(K(z,z)\otimes \Omega(\rho) + I_N\otimes \sigma^2 I_d)$. The latter can be challenging because the matrix is a large, growing matrix. It is an $Nd \times Nd$ matrix and we need to do this at every time step $t$. 

% and we need a sequential approach to do this efficiently, taking note that we are receiving data at different time points. We shall now describe the method to sequentially estimate $\rho$ and compute $\Sigma^*$.

\begin{enumerate}
    \item[(U.1)]The update for $\phi_j(z) | x_{t^j_1}, \ldots, x_{t^j_{j_t}},x_{t+1},\hat{s}_{t^j_1}, \ldots, \hat{s}_{t^j_{j_t}},\\ \hat{s}_{t+1}=j$ in Step 4 of Algorithm~\ref{alg 1} can be computed using the first part of Proposition~\ref{prop:prediction_X}.
\end{enumerate}

Fortunately, we have methods to do both efficiently and sequentially. A key element in the speed up of Algorithm~\ref{alg 1} is fast computation of the matrix inverse in Eq.~\eqref{eq:Posterior_mean_var}. This is carried out as follows. Assume an estimate of $\rho$ as $\rho^{(1)}$ and estimate $\Lambda^{(1)}$. Then, we sequentially estimate $\rho^{(t)}$ and $\Lambda^{(t)}$ by breaking $\Lambda^{(t)}$ into $2 \times 2$ diagonal blocks and use the Schur complement of the block matrix. Since the previous steps store the values of $\Lambda^{(t - 1)}$, the Schur complement needs only compute the block matrix computations relative to the new data points at time $t$. This leads to a massive speed-up in computation and is highlighted in the appendix. An efficient, moment-matching approach to estimate $\rho$ is also discussed there.

\section{Experimental Results}
\label{section:experiments}
In this section we describe the experimental findings of our model and algorithm. We demonstrate the application of our model on simulated multi response data, compare it with the benchmark DP-GP model~\citep{guo2019modeling_dpgp}, and show that it succeeds in both learning the number of hidden Markov states and the transition dynamics. Then, we describe our experiment with real-world traffic data.

\subsection{Simulation Results}
We simulated a dataset using 8 smooth functions $f_1,\dots, f_8$ where each $f_i:\mathbb{R}^2 \to \mathbb{R}^2$. The true functions are shown in the \textbf{appendix}. We also generated a stochastic $8\times 8$ matrix in which each row was generated from a symmetric Dirichlet distribution. We let a Markov chain, $\{s_t\}$, run on the state space $\{1,\dots,8\}$ for $t=1,2,\dots,100$. At each time $t$, we generated $n_t \sim \text{Poi}(100)$ spatial points, $z^{(t)}_1,\dots,z^{(t)}_{n_t} \in [-2,2]\times[-2,2]$. Then, based on $s_t$, we generated $x^{(t)}_1,\dots,x^{(t)}_{n_t}$ as $x^{(t)}_j = f_{s_t}(z^{(t)}_j) + \epsilon$. Here, the $\epsilon$ are independent zero-mean Gaussian random variables with standard deviation, $\sigma=1$. The observed data were $\{(z^{(t)}_i, x^{(t)}_i); i=1,\dots,n_t\}_{t=1}^{100}$. 

\begin{figure}[!t]
\centering     %%% not \center
\subfigure[]{\label{fig:a}\includegraphics[width=40mm]{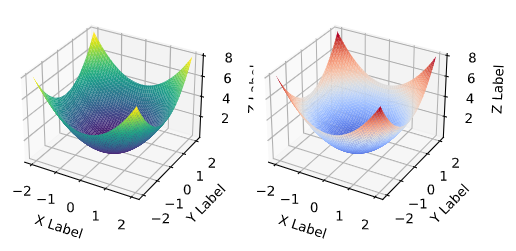}}
\subfigure[]{\label{fig:b}\includegraphics[width=40mm]{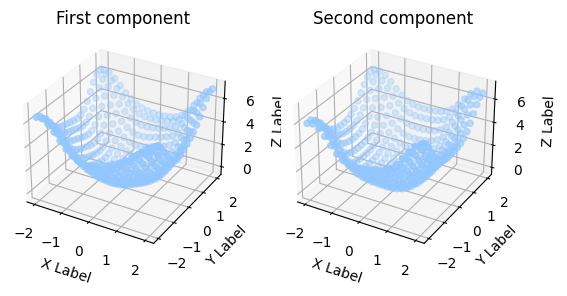}}

\subfigure[]{\label{fig:a}\includegraphics[width=40mm]{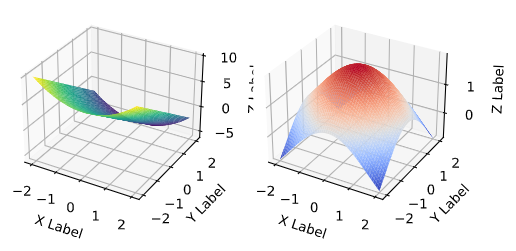}}
\subfigure[]{\label{fig:b}\includegraphics[width=40mm]{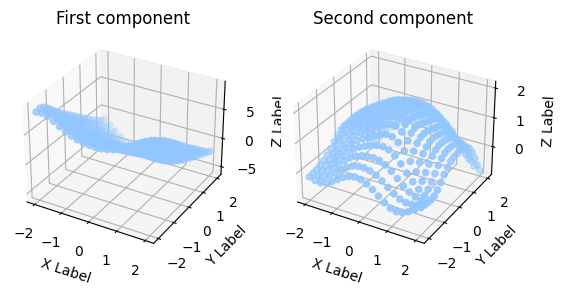}}
\subfigure[]{\label{fig:a}\includegraphics[width=40mm]{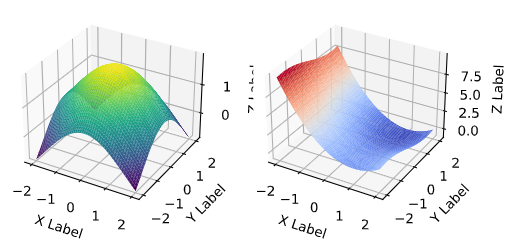}}
\subfigure[]{\label{fig:b}\includegraphics[width=40mm]{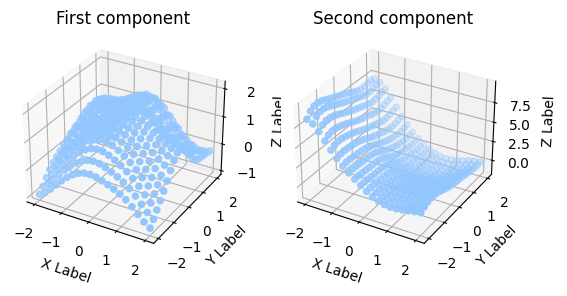}}
\caption{Simulation data: Left column shows the 3 of the true 8 functions (each from $[-2,2]\times [-2,2]\to \mathbb{R}^2)$ while the right column gives the corresponding estimated clusters.\\}
\label{figure}

% \begin{table}[h!]
\centering
\begin{tabular}{||c c c c||} 
 \hline
 $\sigma$ & $K$ & log-lik & time \\ [0.5ex] 
 \hline\hline
 0.2 & 100 & -174603.29 & 139.9 \\ 
 0.5 & 25 & -45290.95 & 126.8 \\
 \textbf{1} & \textbf{8} & \textbf{-29936.35} & \textbf{118.2} \\
 2 & 7 & -36488.81 & 125.0 \\
 5 & 5 & -51801.47 & 195.7 \\ [1ex] 
 \hline
\end{tabular}
\captionof{table}{Performance of Infinite HMM-GP on simulated data.\\}
\label{table:1}
% \end{table}
% \begin{table}[h!]
\centering
\begin{tabular}{||c c c||} 
 \hline
 $\alpha$ & $K$  & time (3 iterations of MCMC) \\ [0.5ex] 
 \hline\hline
 0.2019 & 1 & 196.64 \\ [1ex] 
 \hline
\end{tabular}
\captionof{table}{Performance of DP-GP on simulated data.}
\label{table:2}
% \end{table}
\end{figure}

To fit our model, we estimated the kernel parameters $\sigma_0, \ell_0$ by using the \textit{GPy} package on the data for the first time point. With $\alpha=\gamma=1$, we ran our algorithm in parallel for various values of $\sigma$. The results from various runs of the algorithm are shown in Table \ref{table:1}. Each row in the table shows for a particular $\sigma$, the number of clusters identified ($K$), the final log likelihood of the model (log-lik), and the time in seconds needed to run the algorithm (time). This is a sensible choice because not only are the correct number of clusters identified, the clusters' posterior mean functions are similar to the functions used to generate the data. Table \ref{table:1} also highlights the speed of the algorithm. The algorithm took around 2 minutes for this data set of 10,000 total observations.

Our experiments also demonstrate the inadequacy of DP-GP for this type of data, which underestimates the number of true clusters, and is time-consuming. Figure 1 in the appendix lists the true and estimated clusters for our model.

\subsection{Velocity fields in an LA boulevard}

\begin{figure}[!t]
\centering
\includegraphics[width=0.5\textwidth]{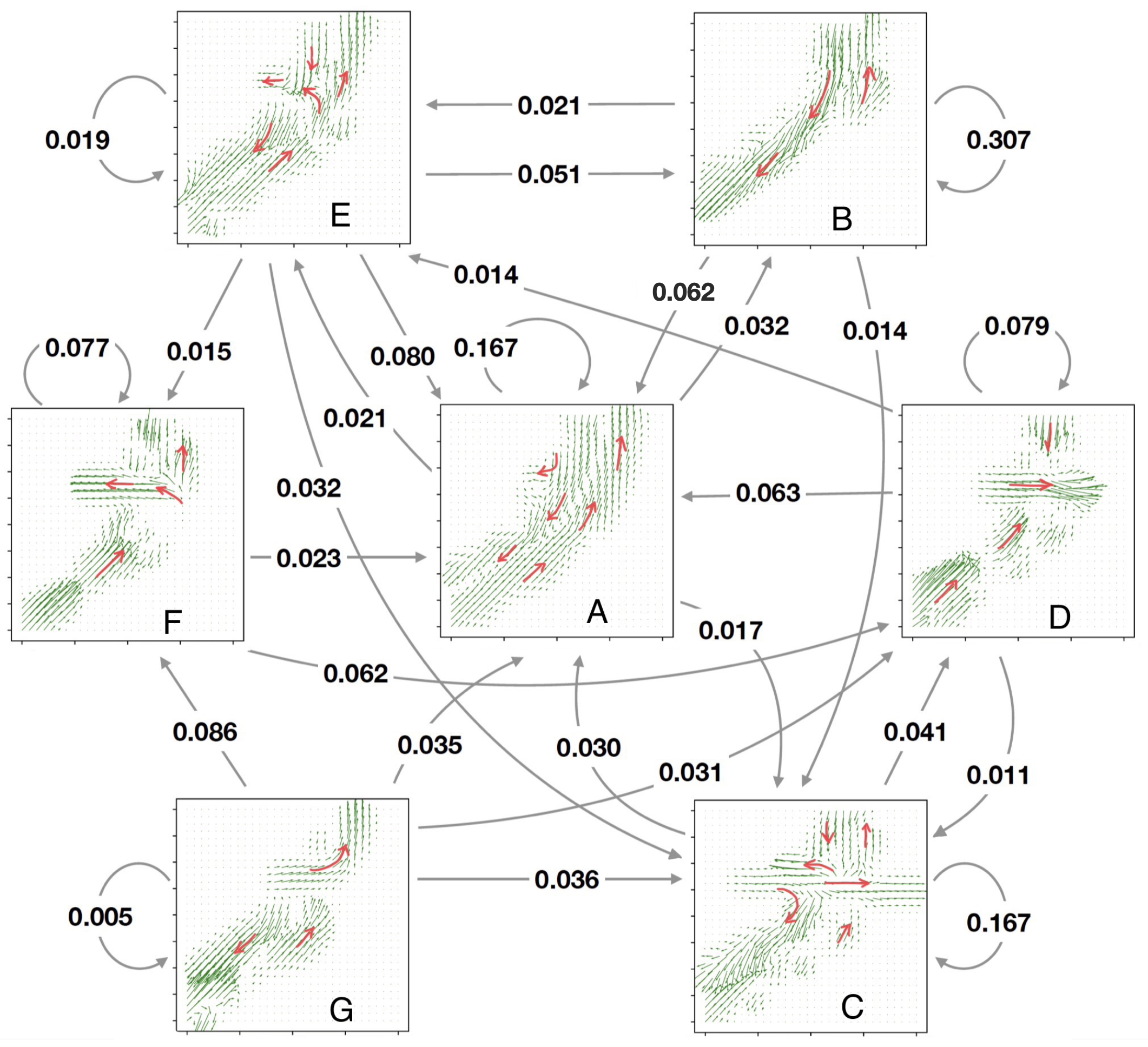}
\caption{Some of the common states along with the estimated 20 step transition probabilities (each transition corresponds to a time gap of 10 seconds). While not generated by the algorithm, the red arrow marks are provided to easily visualize the traffic patterns.}
\label{MC}

\centering     %%% not \center
\subfigure[]{\label{fig:a}\includegraphics[width=40mm]{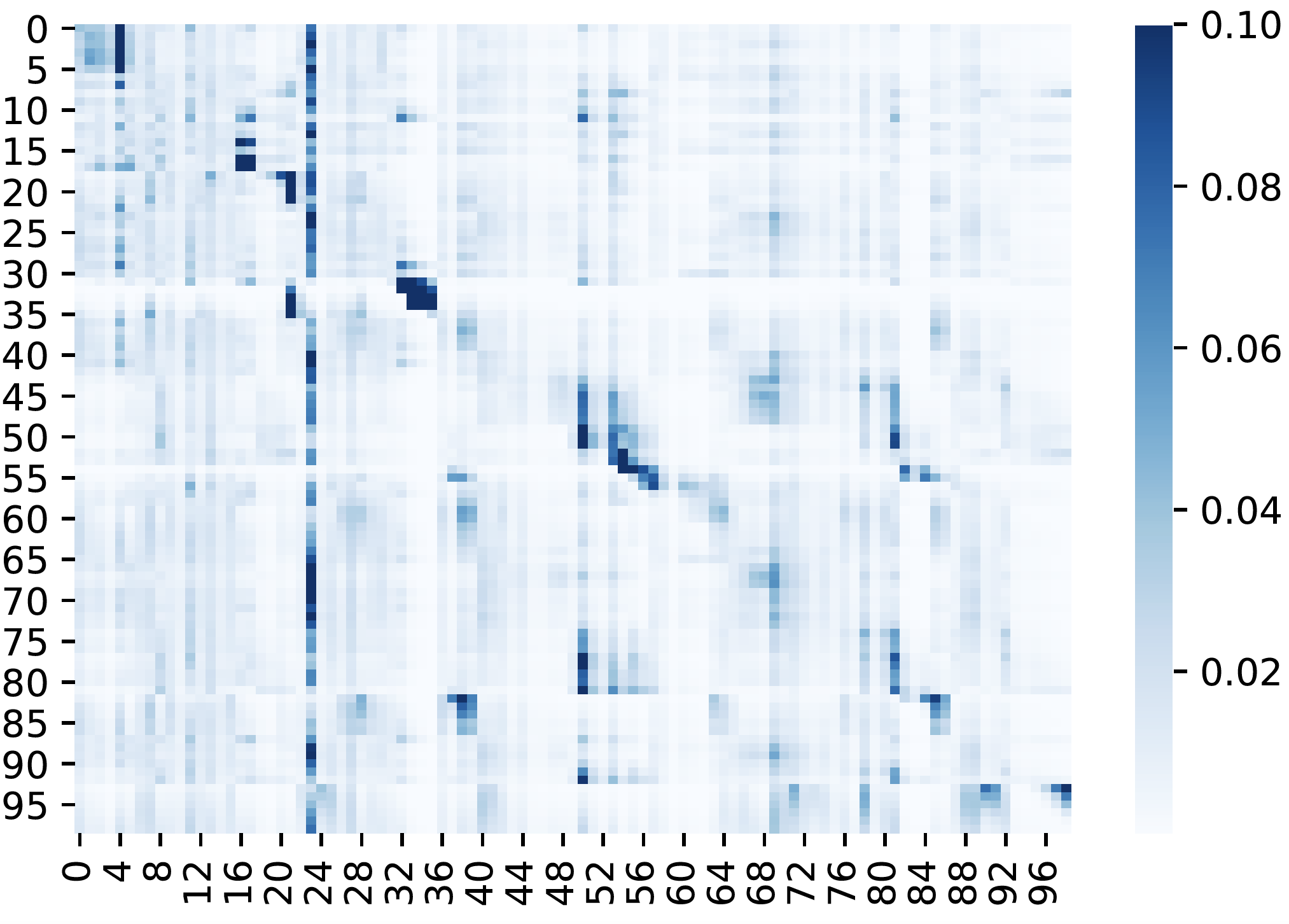}}
\subfigure[]{\label{fig:b}\includegraphics[width=40mm]{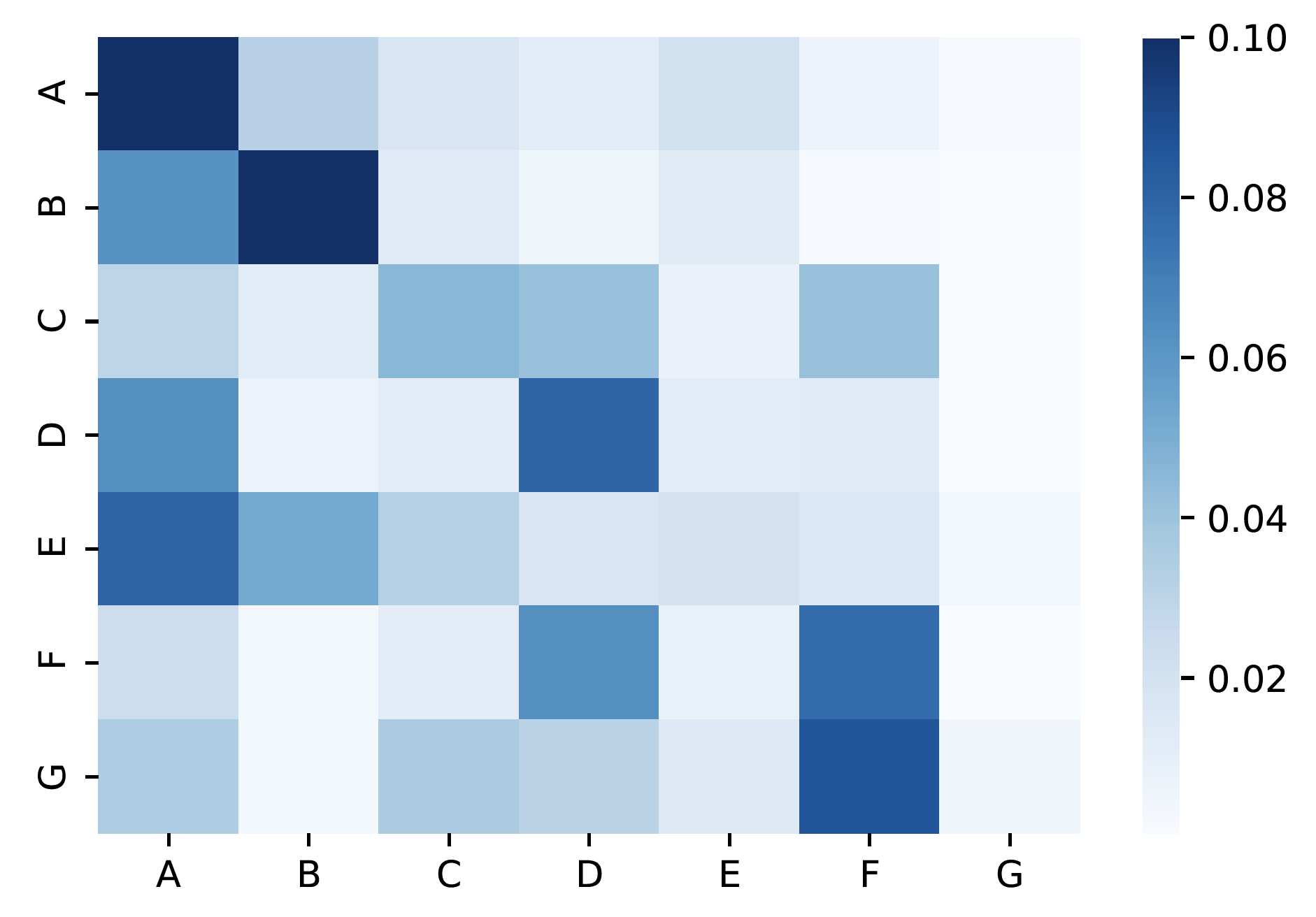}}
\caption{Heatmaps showing the estimated 20-step transition matrix for all hidden states in (a) and for the states considered in figure \ref{MC} in (b)}
\label{heatmap}
\end{figure}

We chose a real-world traffic dataset collected as part of Federal Highway Administration's (FWHA) Next Generation SIMulation (NGSIM) project. The dataset contains detailed multi-vehicle trajectories. As seen in Figure \ref{real}, we focused on data from the intersection of Lankershim Boulevard and Universal Hollywood Dr. in Los Angeles.

After scaling the region into a $[0,1]\times [0,1]$ box, we then discretized this data into \textit{frames} with a duration of 0.5 seconds. Each frame contains the cars' spatial location and velocity separated into the x and y component during that time period. For this study, we took $T=1000$ consecutive frames, which corresponds to roughly 8 minutes. We applied our model and algorithm to extract the latent traffic velocity spatial patterns while also studying the temporal dynamics.

We fixed the infinite HMM hyperparameters at $\alpha=1, \gamma=1$. After we tried different hyperparameters, we found that the log likelihood was highest when we set $\sigma = 0.15$ and $\sigma_0 = \ell_0 = 0.1$. Figure \ref{MC} shows the 7 most commonly occurring patterns that the algorithm identified with this choice of hyperparameters. These clusters are notated so A is the most frequent, B is the second most frequent, and so on. Interestingly enough, it appears that these patterns can be explained by the traffic lights at the intersections. For instance, there appears to be a green light on one or both sides of the vertical road in scenarios A, B, and E. In particular, scenario B differs from the other two because the green light is only for cars coming down the road. This figure also shows the importance of time dynamics. As seen by the moderate transition probabilities from E to C and from C to D and then moving to A, we can see that C is a transitional phase between patterns E (cars turning left from the main highway) and D (cars moving primarily left to right). Scenarios D and E cannot occur together. 

It is also worth mentioning that some of the patterns discovered had a seemingly implausible traffic flow (towards the bottom right) and was later found that the data contained such apparent irregularities. However, upon careful inspection of the map, it was found that there is a driveway in that part of the road towards the right. This explains the apparent discrepancy because cars moving into the driveway would have a velocity directed towards the driveway. It further shows that such models can capture these subtle movements and patterns to give an unsupervised learning about the geography of the roads and the associated traffic patterns. On the other hand, the DP-GP identifies only 13 traffic patterns. It is exciting that we are able to extract such meaningful traffic patterns and understand the temporal dynamics in a fast efficient manner.

% scenarios E and F are different types of left turns from the bottom road

% in the most common scenario A, there is a green light on both vertical roads. Cars coming down the road turn right whereas cars going up do not.

% The red arrow marks are not generated by the algorithm but only to easily visualize the traffic patterns

% We also show the estimated 20-step transition probabilities between these states (very small probabilities are excluded). The red arrow marks are not generated by the algorithm but only to easily visualize the traffic patterns. A is the most prominent state, showing even and smooth flow of traffic along the main highway and is seen to be closely associated with B which is also of a similar nature, although the velocities are higher. There are some interesting observations from this figure. We can see that C is possibly a transitional phase between patterns E (cars turning left from the main highway) and D (cars moving primarily left to right) which cannot occur together and this is reflected by the moderate transition probabilities from E to C and from C to D and then moving to A. Other than this, state E moves directly into A or B mostly. G shows another type of traffic pattern where cars from the left make a turn into the highway. 

% The estimated transition matrix is shown as a heatmap in (a) of Figure \ref{heatmap}. We see that it is very high along the diagonal while extremely scattered otherwise. (b) of the same figure shows the estimated 20-step transition matrix on the states considered in figure \ref{MC}, also as a heatmap. 

\begin{table}[h!]
\centering
\begin{tabular}{|| c | c c c c||} 
 \hline
 $\sigma_0=\ell_0$ & $\sigma$ & $K$ & log-lik & time \\ [0.5ex] 
 \hline\hline
 \multirow{5}{*}{\textbf{0.1}} & 0.1 & 166 & -26461.12 & 3848.9 \\ 
  & \textbf{0.15} & \textbf{99} & \textbf{-2790.70} & \textbf{4467.4} \\
  & 0.2 & 66 & -3878.23 & 5254.7  \\
  & 0.25 & 41 & -12989.42 &  7619.7\\
  & 0.3 & 32 & -23212.65 & 9956.6 \\ [1ex] 
 \hline
\end{tabular}
\caption{Performance of Infinite-HMM-GP on the NGSIM data}
\label{table:3}
\end{table}

\begin{table}[h!]
\centering
\begin{tabular}{||c c c||} 
 \hline
 $\alpha$ & $K$ & time ( 3 iterations)\\ [0.5ex] 
 \hline\hline
 1.8099 & 13 & 4460.7 \\  [1ex] 
 \hline
\end{tabular}
\caption{Performance of DP-GP on the NGSIM data}
\label{table:4}
\end{table}

\section{Conclusion}
We propose a nonparametric framework to model the temporal and spatial aspects of traffic velocity vector field data via the use of Infinite HMM and Gaussian process. Additionally, we provide a fast, efficient sequential, one-pass algorithm for inference that performs MAP estimates of key variables at each step. The model allows us to have a better understanding of traffic movements and reveals interesting temporal patterns in traffic movement that are not captured by other nonparametric models. While the applications for this paper focuses primarily on traffic data, the techniques developed in the paper are also applicable to analyzing more generic spatio-temporal datasets in a fast, efficient manner.

% We focus on the one-step HMM for this paper. It can be extended to the multi-step dependent Markov structure but the analysis becomes increasingly more complex and is beyond the scope of this paper.

\nocite{langley00}

% \bibliography{example_paper}
\bibliographystyle{icml2021}
\bibliography{references,main,Aritra}

\begin{thebibliography}{21}
\providecommand{\natexlab}[1]{#1}
\providecommand{\url}[1]{\texttt{#1}}
\expandafter\ifx\csname urlstyle\endcsname\relax
  \providecommand{\doi}[1]{doi: #1}\else
  \providecommand{\doi}{doi: \begingroup \urlstyle{rm}\Url}\fi

\bibitem[Antoniak(1974)]{Antoniak-74}
Antoniak, C.
\newblock Mixtures of dirichlet processes with applications to bayesian
  nonparametric problems.
\newblock \emph{Annals of Statistics}, 2\penalty0 (6):\penalty0 1152--–1174,
  1974.

\bibitem[Beal et~al.(2002)Beal, Ghahramani, and Rasmussen]{beal2002infinite}
Beal, M.~J., Ghahramani, Z., and Rasmussen, C.~E.
\newblock The infinite hidden markov model.
\newblock In \emph{Advances in neural information processing systems}, pp.\
  577--584, 2002.

\bibitem[Blei \& Jordan(2006)Blei and Jordan]{Blei-VI-DPMM}
Blei, D. and Jordan, M.
\newblock Variational inference for dirichlet process mixtures.
\newblock \emph{Bayesian Analysis}, 1:\penalty0 121--144, 2006.

\bibitem[Blei et~al.(2003)Blei, Ng, and Jordan]{Blei-et-al}
Blei, D., Ng, A., and Jordan, M.
\newblock Latent {D}irichlet allocation.
\newblock \emph{J. Mach. Learn. Res}, 3:\penalty0 993--1022, 2003.

\bibitem[Chen et~al.(2016)Chen, Liu, Liu, Miller, and How]{Pedestrian_BNP}
Chen, Y.~F., Liu, M., Liu, S., Miller, J., and How, J.
\newblock Predictive modeling of pedestrian motion patterns with bayesian
  nonparametrics.
\newblock \emph{AIAA 2016-1861}, 2016.

\bibitem[Cressie(1993)]{Cressie-93}
Cressie, N.
\newblock \emph{Statistics for Spatial Data}.
\newblock Wiley, NY, 1993.

\bibitem[Ferguson(1973)]{Ferguson-73}
Ferguson, T.
\newblock A {B}ayesian analysis of some nonparametric problems.
\newblock \emph{Ann. Statist.}, 1:\penalty0 209--230, 1973.

\bibitem[Foti et~al.()Foti, Xu, Laird, and Fox]{SVI-HMM}
Foti, N., Xu, J., Laird, D., and Fox, E.~B.
\newblock Stochastic variational inference for hidden markov models.

\bibitem[Fox et~al.(2009)Fox, Sudderth, Jordan, and Willsky]{Fox-etal-09}
Fox, E., Sudderth, E., Jordan, M.~I., and Willsky, A.
\newblock The sticky hdp-hmm: {B}ayesian nonparametric hidden {M}arkov models
  with persistent states.
\newblock Technical Report P-2777, MIT LIDS, 2009.

\bibitem[Fox et~al.(2011)Fox, Sudderth, Jordan, and Willsky]{Fox-etal-11}
Fox, E.~B., Sudderth, E.~B., Jordan, M.~I., and Willsky, A.~S.
\newblock A sticky hdp-hmm with application to speaker diarization.
\newblock \emph{Annals of Applied Statistics}, 5 : 2A:\penalty0 1020--1056,
  2011.

\bibitem[Gelfand \& Smith(1990)Gelfand and Smith]{Gelfand-MCMC}
Gelfand, A. and Smith, A.
\newblock Sampling-based approaches to calculating marginal densities.
\newblock \emph{Journal of the American Statistical Association}, 85
  (410):\penalty0 398–409, 1990.

\bibitem[Ghosal \& van~der Vaart(2017)Ghosal and van~der
  Vaart]{Ghosal-VDV-BNP-book}
Ghosal, S. and van~der Vaart, A.
\newblock \emph{Fundamentals of nonparametric Bayesian inference, vol. 44 of
  Cambridge Series in Statistical and Probabilistic Mathematics.}
\newblock Cambridge University Press, Cambridge, 2017.

\bibitem[Guo et~al.(2019)Guo, Kalidindi, Arief, Wang, Zhu, Peng, and
  Zhao]{guo2019modeling_dpgp}
Guo, Y., Kalidindi, V.~V., Arief, M., Wang, W., Zhu, J., Peng, H., and Zhao, D.
\newblock Modeling multi-vehicle interaction scenarios using gaussian random
  field.
\newblock \emph{arXiv preprint arXiv:1906.10307}, 2019.

\bibitem[Hoffman et~al.(2013)Hoffman, Blei, Wang, and
  Paisley]{hoffman2013stochastic}
Hoffman, M.~D., Blei, D.~M., Wang, C., and Paisley, J.
\newblock Stochastic variational inference.
\newblock \emph{Journal of Machine Learning Research}, 14\penalty0
  (1):\penalty0 1303--1347, May 2013.

\bibitem[Horn et~al.(1994)Horn, Horn, and Johnson]{horn1994topics}
Horn, R.~A., Horn, R.~A., and Johnson, C.~R.
\newblock \emph{Topics in matrix analysis}.
\newblock Cambridge university press, 1994.

\bibitem[Jordan et~al.(1999)Jordan, Ghahramani, Jaakkola, and
  Saul]{jordan1999introduction}
Jordan, M.~I., Ghahramani, Z., Jaakkola, T.~S., and Saul, L.~K.
\newblock An introduction to variational methods for graphical models.
\newblock \emph{Machine Learning}, 37\penalty0 (2):\penalty0 183--233, 1999.

\bibitem[Joseph et~al.(2011)Joseph, Doshi-Velez, Huang, and
  Roy]{joseph2011bayesian_dpgp}
Joseph, J., Doshi-Velez, F., Huang, A.~S., and Roy, N.
\newblock A bayesian nonparametric approach to modeling motion patterns.
\newblock \emph{Autonomous Robots}, 31\penalty0 (4):\penalty0 383, 2011.

\bibitem[Kim et~al.(2011)Kim, Lee, and Essa]{Kim:2011:Gaussian-Process}
Kim, K., Lee, D., and Essa, I.
\newblock Gaussian process regression flow for analysis of motion trajectories.
\newblock In \emph{Proceedings of IEEE International Conference on Computer
  Vision (ICCV)}. IEEE Computer Society, November 2011.

\bibitem[Mandt et~al.(2017)Mandt, Hoffman, and Blei]{mandt2017stochastic}
Mandt, S., Hoffman, M.~D., and Blei, D.~M.
\newblock Stochastic gradient descent as approximate {B}ayesian inference.
\newblock \emph{Journal of Machine Learning Research}, 18\penalty0
  (134):\penalty0 1--35, 2017.

\bibitem[Rabiner(1989)]{Rabiner}
Rabiner, L.
\newblock A tutorial on hidden markov models and selected applications in
  speech recognition.
\newblock \emph{Proceedings of the IEEE}, 77:\penalty0 257--–285, 1989.

\bibitem[Teh et~al.(2006)Teh, Jordan, Beal, and Blei]{Teh-etal-06}
Teh, Y., Jordan, M., Beal, M., and Blei, D.
\newblock Hierarchical {D}irichlet processes.
\newblock \emph{J. Amer. Statist. Assoc.}, 101:\penalty0 1566--1581, 2006.

\end{thebibliography}

\clearpage
\newpage
\appendix 
\section{Computations for multivariate response data}
\subsection{Proof of proposition \ref{prop:prediction_X}}

\begin{proof}
The proof of this proposition follows from the matrix normal assumption and the conditional normal formula. By the assumption we have
\begin{equation}
    \begin{pmatrix}
    \text{vec}\left((x^{t^j_k})\mid_{k=1:j_t}\right) \\
    \phi_j(z)
    \end{pmatrix}\, \sim \,\mathcal{N}_{d(N+1)}\Bigg(\boldsymbol{0}, \begin{bmatrix}
    K_x & K_{x,\phi} \\
    K_{x,\phi}^T & K_\phi
    \end{bmatrix}\Bigg),
\end{equation}
where
$$K_x = K(Z_{t}^j,Z_{t}^j)\otimes \Omega(\rho) + I_N\otimes \sigma^2 I_d$$
$$K_{x,\phi} = K(z,Z_{t}^j)\otimes \Omega(\rho)$$
$$K_{\phi} = K(z,z) \otimes \Omega(\rho)$$
The proposition follows by using the formula for the conditional distribution of Gaussian random variable. The proposition can be easily extended to cover the case where we try to determine the posterior of $\phi_j$ at any $z_1,\dots,z_m \in \mathbb{R}^p$, simply by vectorizing it and using the multivariate version of Lemma~\ref{lemma:univariate} in section~\ref{ssection:univariate}.

\end{proof}
\subsection{Update for state variable }
\paragraph{second term in RHS of Eq.11:}
Here, we provide the computation of the second term in the RHS of Eq.(11) in the main draft.
Let $t^j_1,\ldots,t^j_{j_t} \leq t$ denote all indices less than $t+1$ for which $s_t=j$. Then, 
\begin{eqnarray}
 & &\Prob(X_{t+1}(z^{t+1}_{1:n^{(t+1)}}) \mid s_{t+1}=j,\mathcal{H}_t) \nonumber \\
 &=&\int_{\Theta}  F(X_{t+1}(z^{t+1}_{1:n^{(t+1)}}); \phi_j)dG(\phi | \{X_u: u\in T\}) \nonumber\\
 &=&\Prob(X_{t+1}(z^{t+1}_{1:n^{(t+1)}}) \mid X_{t^j_1},\ldots,X_{t^j_{j_t}}).\label{eq: X_predictive}
\end{eqnarray}
Here, $G$ is the $MRGP(\mu,K,\rho)$ distribution. Moreover,
\begin{eqnarray}
      & &F(X_{t+1}(z^{t+1}_{1:n^{(t+1)}}); \phi_j)\nonumber \\  &=&\Prob(X_{t+1}(z^{t+1}_{1:n^{(t+1)}})\mid s_{t+1}=j, \{\phi_k\}). 
  \label{eq: complete_lik}
\end{eqnarray}

The last line of Eq.~\eqref{eq: X_predictive} is then calculated using Prop.~\ref{prop:prediction_X}.

\subsection{Update for oracle variable}
\begin{lemma}
 $X_{t+1}(z^{t+1}_{1:n^{(t+1)}})$ is independent of $o_{t+1}$ conditioned on $s_{t+1}$,  i.e.,
 \begin{eqnarray}
     & &\Prob(X_{t+1}(z^{t+1}_{1:n^{(t+1)}}) \mid s_{t+1} = j, o_{t+1}, \mathcal{H}_t) \nonumber\\ &=& \Prob(X_{t+1}(z^{t+1}_{1:n^{(t+1)}}) \mid s_{t+1} = j, \mathcal{H}_t)
 \end{eqnarray}
\end{lemma}

\begin{proof}

First by Bayes' Rule,
\begin{eqnarray}
    & &\Prob(o_{t+1} =  e\mid X_{t+1}(z^{t+1}_{1:n^{(t+1)}}), s_{t+1}= j, \mathcal{H}_t)  \nonumber \\ 
    & &\propto \Prob(o_{t+1} =  e\mid s_{t+1}= j, \mathcal{H}_t)\nonumber\\ & & \hspace {2 em} \Prob(X_{t+1}(z^{t+1}_{1:n^{(t+1)}})\mid s_{t+1} = j, \mathcal{H}_t)
\end{eqnarray}
Since $\Prob(X_{t+1}(z^{t+1}_{1:n^{(t+1)}}) \mid s_{t+1} = j, \mathcal{H}_t)$ is free of $o_{t+1}$, it gets cancelled when normalized. Thus we obtain for every $j = \{1,\dots, \tilde{K}^{(t)}+1\}$ and $e\in \{0,1\}$,
\begin{align}
    \Prob(o_{t+1} =  e\mid X_{t+1}(z^{t+1}_{1:n^{(t+1)}}), s_{t+1}= j, \mathcal{H}_t) \\ \nonumber = \Prob(o_{t+1} =  e\mid s_{t+1}= j, \mathcal{H}_t)
    \label{oracle_post}
\end{align}
which reflects the fact that given $s_{t+1}$, $o_{t+1}$ does not depend on $X_{t+1}(z^{t+1}_{1:n^{(t+1)}})$. 
\end{proof}

\section{Prior literature}
In this section we briefly introduce the various tools used in the paper.
\subsection{Hierarchical Dirichlet Process}
The HDP is a bayesian non parametric prior which enables us to fit mixture model for each group in a grouped data while allowing the mixtures to share components. Suppose there are $J$ groups, then HDP is a distribution over a set of random probability measures over $(\Theta, \mathcal{B})$; one $G_j$ for the $j$th group and a global probability measure $G_0$.

Suppose the $j$th group consists of $n_j$ datapoints $(x_{j1},\dots, x_{jn_j})$. The HDP model is as follows:
\begin{align}
    G_0|\gamma, H &\sim DP(\gamma, H) \\
    G_j|\alpha, G_0 &\sim DP(\alpha, G_0) \quad ,j = 1,2\dots,J \\
    \theta_{ji}|G_j &\sim G_j \quad ,j=1,2\dots, J \\
    x_{ji}|\theta_{ji} &\sim F(.|\theta_{ji}) \quad ,i=1,\dots, n_j; j=1,\dots,J
\end{align}

The parameters include $\gamma$, $\alpha$ and $H$. $\theta$ are the latent factors in the model and $F$ is the kernel. $G_j$'s are conditionally independent given $G_0$ and given $G_j$, $\theta_{j1},\dots, \theta_{jn_j}$ are iid. To see how this model captures sharing of mixture components, we look at the stick breaking construction of the DP and find that $G_0$ is atomic and
$$G_0 = \sum_k \beta_k\delta_{\phi_k}$$
Also, by construction of $G_j$,
$$G_j = \sum_{k} \pi_{jk}\delta_{\phi_k}$$
which shows that the atoms of $G_j$ originate from those of $G_0$ (and are hence shared across groups). Thus identifying $\phi_k$ as the parameter for the $k$th mixture component, we find that each of the $J$ groups are modelled as mixture distributions with the same set of (countably infinite) mixture components, but have different mixing proportions, given by $\pi_j = (\pi_{jk})_{k=1}^\infty$.

\subsection{HDP-HMM}
This is the model described in section (7) of \citep{Teh-etal-06}. It uses the $\pi_j$ for both transitions and emissions.

As described in (3.1) in~\citep{Teh-etal-06}, the $\pi_j$ and the atoms $\theta_k$ can be generated equivalently as follows:
$$\beta|\gamma\sim GEM(\gamma) \quad\quad \pi_j|\alpha,\beta \sim DP(\alpha, \beta)\ \quad \theta_k|H \sim H$$
The $\pi_j=(\pi_{jk})_{k=1}^\infty$ denoted the mixture probabilities for the $j$th group over the atoms $\theta=(\theta_k)_{k=1}^\infty$.

The model discussed in (7) in the same paper extends this to the HDP-HMM model which is as follows: there are countably infinite states (each representing a mixture) and all the mixtures share the same atoms/components. The hidden state indicates the component with transition given by the corresponding row of $\pi$ (now a doubly infinite stochastic matrix) and the emission is given as before. In particular, the model consists of the following (note $\beta$ is a probability distribution over $\mathbb{N}$)
\begin{align*}
    \beta|\gamma &\sim GEM(\gamma) \\
    \pi_k|\alpha,\beta &\sim DP(\alpha,\beta) \\
    \theta_k|H &\sim H
\end{align*}
and the associated HMM is given by:
\begin{align*}
    s_t|s_{t-1},\{\pi_k\}_{k=1}^\infty &\sim \pi_{s_{t-1}} \\
    x_t|s_t, \{\theta_k\}_{k=1}^\infty &\sim F(.|\theta_{s_t})
\end{align*}

\subsection{Matrix Normal Distribution}
Let $X\in\mathbb{R}^{m\times n}$ be a random matrix valued random variable. We say that $X$ follows a matrix normal distribution ($MN_{mn}$ in short) with mean parameter $M\in\mathbb{R}^{m\times n}$ and scale parameters $U \in \mathbf{S}^m_{++}$ and $V\in\mathbf{S}^n_{++}$ if the pdf is 
\begin{eqnarray}
   & & p(X|M,U,V) \nonumber \\ &=& \frac{\exp\left(-\frac{1}{2}\text{tr}\left[(V^{-1}(X-M)^TU^{-1}(X-M)\right]\right)}{(2\pi)^{mn/2}|U|^{m/2}|V|^{n/2}}
\end{eqnarray}
We write this as
\begin{eqnarray}
& &X \sim MN_{mn}(M,U,V) \nonumber \\ &\iff& \text{vec}(X) \sim N_{mn}\left(\text{vec}(M), U \otimes V\right)
\end{eqnarray}
which establishes its connection with the multivariate normal distribution. Here $\text{vec}$ indicates vectorized form of the corresponding matrix (we define it as vector obtained by stacking the rows of the matrix on top of each other) and $\otimes$ is the Kronecker product.

We list a few properties of this matrix normal distribution which follow readily using the equivalent multivariate normal form.
\begin{enumerate}
    \item Mean: $\mathbb{E}(X) = M$ 
    \item Second order moments: \begin{align*}
    \mathbb{E}\left[(X-M)(X-M)^T\right] = U\text{tr}(V),\\ \mathbb{E}\left[(X-M)^T(X-M)\right] = V\text{tr}(U)
    \end{align*}
    \item For appropriate sized matrices $A,B,C$ we have
    \begin{align*}
        \mathbb{E}\left[XAX^T\right] &= U\text{tr}(A^TV) + MAM^T \\
        \mathbb{E}\left[X^TBX\right] &= V\text{tr}(UB^T) + M^TBM \\
        \mathbb{E}\left[XCX\right] &= VC^TU + MCM 
    \end{align*}
    \item If $X \sim MN_{mn}(M,U,V)$ then
    \begin{align*}
        X^T &\sim MN_{nm}(M^T,V,U) \\
        DXC &\sim MN_{rs}(DMC, DUD^T, C^TVC)
    \end{align*}
    where$D\in\mathbb{R}^{r\times m}$ has rank $r\leq m$ and $C\in\mathbb{R}^{n\times s}$ has rank $s\leq n$.
    \item Maximum likelihood estimation: Let $X_1,\dots,X_k \overset{iid}{\sim} MN_{mn}(M,U,V)$, then the MLE of $M$ has a closed form solution:
    $$\hat{M} = \frac{1}{k}\sum_{j=1}^k X_k$$
    However $U$ and $V$ do not have MLE in closed form but they satisfy:
    \begin{align*}
        U &= \frac{1}{kn}\sum_{j=1}^k (X_j-M)V^{-1}(X_j-M)^T \\
        V &= \frac{1}{km}\sum_{j=1}^k (X_j-M)^TU^{-1}(X_j-M)
    \end{align*}
    The estimates are positive definite if $k \geq \max \{m/n, n/m\} + 1$. Also they are identifiable upto a scalar multiple, i.e. $MN_{mn}(M,U,V) = MN_{mn}(M,sU,(1/s)V)$
\end{enumerate}

\section{Calculations for posterior computation}
In this section we consider the calculations for univariate response data. The results for the multivariate case follow similarly.
\subsection{Univariate response data}
\label{ssection:univariate}
Consider the case of univariate response data. i.e. every cluster is a function $\phi_j:\mathbb{R}^p \to\mathbb{R}$ and we place a usual Gaussian process prior on them. We also consider isotropic data in this section, i.e. for each time point we observe noisy observations around the true cluster functions at the same spatial locations $z_1,\dots,z_N\in\mathbb{R}^p$. We  describe computing the posterior predictive distribution of $\phi_k$ at time $t+1$ for some $t \in {1, 2 \ldots, T}$, $k \in {1, 2, \ldots, K^{(t)}}$, and $N \in \mathbbm{N}$.

Recall that given $s_t=j$, the observation at time $t$ are given by 
\begin{eqnarray}
    X_t(z_i) = \phi_j(z_i) + \epsilon_{ji},
\end{eqnarray}
$i=1,\dots,N$ and $\epsilon$'s are iid $\mathcal{N}(0,\sigma^2)$. 
The posterior computation for the latent function, $\phi_j$ can be obtained by the following lemma.

\begin{lemma}
\label{lemma:univariate}
Let $t_1, t_2, \ldots, t_{N_k^{(t)}}$, $N_k^{(t)} \in \mathbbm{N}$, indicate the times during which $\widehat{s}_{t'} = k$ for $t' \leq t$. Let us define $X_t = (X_t(z_1),\dots, X_t(z_N))^T\in\mathbb{R}^N$, $z=(z_1,\dots,z_N)$ and $\phi_k(z) = (\phi_k(z_1),\dots,\phi_k(z_N))^T$. Moreover, let $K$ be the $N\times N$ RBF kernel matrix over $(z_1,\dots,z_N)$, i.e. $K_{ij} = K(z_i,z_j)$ for $i,j \in \{1,\dots,N\}$. Then, the posterior of $\phi$ conditioned on the data is given by:
\begin{equation}
    \begin{aligned}
    \phi_k(z) &| X_{t_1}, X_{t_2}, \ldots, X_{t_{N_k^{(t)}}} \sim\\ 
    & \mathcal{N}_{N}\Big(K_{x,\phi}^T K_{x}^{-1}\boldsymbol{x}, K_\phi - K_{x,\phi}^T K_{x}^{-1} K_{x,\phi}\Big).
    \end{aligned}
\end{equation}
where $\boldsymbol{x}$ is the $NN_k^{(t)}$ length vector obtained by stacking $X_{t_1},\dots, X_{t_{N_k^{(t)}}}$ and
$$K_\phi = K, \quad K_{x,\phi} = \begin{bmatrix}
    K \\
    \vdots \\
    K
\end{bmatrix}_{N N^{(t)}_k \times N},$$
$$K_x = \begin{bmatrix}
K+\sigma^2I & K & \dots & K \\
K & K+\sigma^2I & \dots & K \\
\vdots & \vdots & \ddots & \vdots \\
K & K & \dots & K + \sigma^2I
\end{bmatrix}_{N N^{(t)}_k \times N N^{(t)}_k}.$$

Moreover, 
% the posterior predictive distribution of $\phi_k$ for $X_{t}$ (which is used in estimating the state assignment) given the data and the estimated state assignment to $k$ would then be given by
 
\begin{equation}
\label{eq:pred_uni}
    \begin{aligned}
    X_t &| \hat{s}_t=k, \mathcal{H}_{t-1} \sim\\ 
    & \mathcal{N}_{N}\Big(K_{x,\phi}^T K_{x}^{-1}\boldsymbol{x}, K_\phi - K_{x,\phi}^T K_{x}^{-1} K_{x,\phi} + \sigma^2 I_N\Big).
    \end{aligned}
\end{equation}

\end{lemma}

\begin{proof}
By normality, we have that

\begin{equation*}
    \begin{pmatrix}
    X_{t_1} \\
    X_{t_2} \\
    \vdots \\
    X_{t_{N_k^{(t)}}} \\
    \phi_k(z)
    \end{pmatrix}\, \sim \,\mathcal{N}_{N(N_k^{(t)}+1)}\Bigg(\boldsymbol{0}, \begin{bmatrix}
    K_x & K_{x,\phi} \\
    K_{x,\phi}^T & K_\phi
    \end{bmatrix}\Bigg),
\end{equation*}

The result now follows by using the formulation of the conditional normal distribution. The proof of Eq.~\eqref{eq:pred_uni} also follows similarly.
\end{proof}

Notice that the kernel $K$ stays the same for every time point because $z_1,\dots,z_N$ are fixed in this case.

The inversion of matrix $K_x$ is computationally intensive as $K_x$ is $N N^{(t)}_k\times N N^{(t)}_k$. Moreover, $K_x$, which is $N N^{(t)}_k\times N N^{(t)}_k$ and grows as new points are included. Therefore, to circumvent the issue of inverting the matrix, we can instead use the spectral decomposition of $K$. This method is described in the following lemma.

\begin{lemma}
 Suppose that $K$ has eigenvalues $\lambda_1, \lambda_2, \ldots, \lambda_{N}$ with corresponding eigenvectors $v_1, v_2, \ldots, v_{N} \in \mathbbm{R}^{N}$. For $U \in \mathbbm{R}^{N N^{(t)}_k \times N}$, let $U = [u_1, u_2, \ldots, u_{N^{(t)}_k}]$ where $u_n = \mathbf{1}_{N^{(t)}_k} \otimes v_n/\sqrt{N^{(t)}_k}$, $n = 1, 2, \ldots, N$. Set $D$ to be a diagonal $N\times N$ matrix such that the $n$th diagonal entry is given by $-\frac{k\lambda_i}{\sigma^2(k\lambda_i + \sigma^2)}$. Then,
\begin{align}
    K_x^{-1} = UDU^T + \frac{1}{\sigma^2}I_{N N^{(t)}_k}.
\end{align}
%Need to explain why this method saves time by having location
% be the matrix $[u_1, \dots, u_n]$ with 
% % (note we only need to spectral decompose $K$ once), 
% then using the $nk\times n$ matrix $U = [u_1, \dots, u_n]$, $U \in \mathbbm{R}^{nk \times n}$ with $u_i = \mathbf{1}_k \otimes v_i/\sqrt{k}$ and the $n\times n$ diagonal matrix $D$ with the $i$th diagonal entry given by $-\frac{k\lambda_i}{\sigma^2(k\lambda_i + \sigma^2)}$ and then we can compute $$K_x^{-1} = UDU^T + \frac{1}{\sigma^2}I_{nk}$$
\end{lemma}

% Note that we also can use this lemma to calculate the posterior predictive distribution when we are estimating $s_t$ at time $t$.
% Using this we can not only obtain the posterior distribution of $\phi(z)$ but also the posterior predictive distribution which is used in step 1 of the algorithm.

\begin{proof}
We use the following result corresponding to Kronecker product (Theorem 4.2.12 in \citep{horn1994topics}) to ease this computation.

\begin{theorem}[Eigendecomposition: Kronecker Product]
Suppose $A\in\mathbb{M}^n$ and $B\in\mathbb{M}^m$. Let $\lambda$ be an eigenvalue of $A$ with corresponding eigenvector $x$ and $\mu$ be an eigenvalue of $B$ with corresponding eigenvector $y$. Then $\lambda\mu$ is an eigenvalue of $A\otimes B$ with corresponding eigenvector $x\otimes y$. Any eigenvalue of $A \otimes B$ arises as such a product of eigenvalues of A and B.
\label{thm:eigen}
\end{theorem}

Recall that $K_x = A + \sigma^2 I$ where $A=\boldsymbol{1}_{N^{(t)}_k}\boldsymbol{1}_{N^{(t)}_k}^T \otimes K$.

\textbf{Step 1:} We first provide the eigendecomposition of $K_x$. 

Consider the eigendecomposition of $A$. 

Let $K = \sum_{i=1}^N \lambda_i v_iv_i^T$ be the eigendecomposition of $K$ (note that $K$ is fixed throughout and this decomposition needs to be done once).

Also, $\boldsymbol{1}_{N^{(t)}_k}\boldsymbol{1}_{N^{(t)}_k}^T = N^{(t)}_k \frac{\boldsymbol{1}_{N^{(t)}_k}}{\sqrt{N^{(t)}_k}}\frac{\boldsymbol{1}_{N^{(t)}_k}}{\sqrt{N^{(t)}_k}}^T$ gives the corresponding decomposition for $\boldsymbol{1}_{N^{(t)}_k}\boldsymbol{1}_{N^{(t)}_k}^T$.

Thus the eigendecomposition of $A$ by Lemma \ref{thm:eigen} is 
$$A = \sum_{i=1}^N (N^{(t)}_k\lambda_i) u_i u_i^T \quad\text{where } u_i = \boldsymbol{1}_{N^{(t)}_k}\otimes v_i/\sqrt{N^{(t)}_k}$$
showing that $A$ has only $N$ non-zero eigenvalues with corresponding eigenvectors $u_1,\dots, u_N$. 

Now extend $\{u_1,\dots, u_N\}$ to an orthonormal basis of $\mathbb{R}^{N N^{(t)}_k}$, as $\{u_1,\dots, u_N, u_{N+1}, \dots, u_{N N^{(t)}_k})$ (e.g. Gram Schmidt). 

Then we can write $K_x$ as
\begin{eqnarray}
    K_x &=& A + \sigma^2 I = \sum_{i=1}^N (N^{(t)}_k \lambda_i) u_iu_i^T + \sum_{j=1}^{N N^{(t)}_k} \sigma^2 u_ju_j^T \nonumber \\ & =&  \sum_{i=1}^N (N^{(t)}_k\lambda_i+\sigma^2)u_iu_i^T + \sum_{i>N} \sigma^2 u_iu_i^T
\end{eqnarray}
which shows that $K_x$ has eigenvalues $N^{(t)}_k\lambda _1+\sigma^2,\dots,N^{(t)}_k\lambda_N+\sigma^2$ with multiplicity 1 and $\sigma^2$ with multiplicity $N N^{(t)}_k - N$.

\textbf{Step 2:} From the eigendecomposition of $K_x$, the decomposition for $K_x^{-1}$ is given by
\begin{align}
    K_x^{-1}  &= \sum_{i=1}^N \frac{1}{N^{(t)}_k\lambda_i+\sigma^2} u_iu_i^T + \sum_{i>N} \frac{1}{\sigma^2} u_iu_i^T \nonumber \\
    &= \sum_{i=1}^N \Bigg(\frac{1}{N^{(t)}_k\lambda_i+\sigma^2} - \frac{1}{\sigma^2}\Bigg) u_iu_i^T + \frac{1}{\sigma^2} I_{N N^{(t)}_k} \nonumber\\
    &= -\frac{1}{\sigma^2}\sum_{i=1}^N \Bigg(\frac{N^{(t)}_k \lambda_i}{N^{(t)}_k\lambda_i + \sigma^2}\Bigg) u_iu_i^T + \frac{1}{\sigma^2} I_{N N^{(t)}_k}
\end{align}

The first term can be written as $UDU^T$, where, 
$$D = \frac{1}{\sigma^2}\begin{bmatrix}
    -\frac{N^{(t)}_k\lambda_1}{N^{(t)}_k\lambda_1+\sigma^2} & 0 & \dots & 0\\
    0 & -\frac{N^{(t)}_k\lambda_2}{N^{(t)}_k\lambda_2+\sigma^2} & \dots & 0 \\
    \vdots & \vdots & \ddots & \vdots \\
    0 & 0 & \dots & -\frac{N^{(t)}_k\lambda_n}{N^{(t)}_k\lambda_n+\sigma^2}
\end{bmatrix}$$
This proves the result.
\end{proof}

During initialization, we spectral decompose $K$ to get $\lambda_1,\dots, \lambda_N$ and $v_1,\dots, v_N$ and store them. Then, when needed as in our discussion above, we construct the $N N^{(t)}_k\times N$ matrix $U = \begin{bmatrix}u_1 & u_2 & \dots & u_N\end{bmatrix} $ and  a $N\times N$ diagonal matrix $D$ and then use $K_x^{-1} = UDU^T + \frac{1}{\sigma^2}I_{N N^{(t)}_k}$ to get the required inverse.

Note that we can further speed up computation using the above decomposition because we only have to compute the eigenvalues of $K$ once. $K_x$ is a $N N^{(t)}_k\times N N^{(t)}_k$ matrix and its Cholesky decomposition has complexity  $O(N^3 (N^{(t)}_k)^3/3)$, which is quite large for $N^{(t)}_k\approx 10^4$ and $N\approx 100$ and considering that we need to do this for every time point $t$. 

% \subsection{MRGP}

% \begin{prop}
% \label{prop:prediction_X}
% Given notations in (P.1)-(P.4) and $\hat{s}_{t^j_1}, \ldots, \hat{s}_{t^j_{j_t}}=j$, we have that for any $z \in \mathbb{R}^p$.
% \begin{eqnarray}
%     & &\text{vec}(\phi_j(z)) | x_{t^j_1}, \ldots, x_{t^j_{j_t}},\hat{s}_{t^j_1}, \ldots, \hat{s}_{t^j_{j_t}}=j\nonumber \\ &\sim &\mathcal{N}_{Nd}\left(\mu^*, \Sigma^*\right)
% \end{eqnarray}
% where
% \begin{align}
% \label{eq:Posterior_mean_var}
%     \Lambda &= \left(K(Z_{t}^j,Z_{t}^j)\otimes \Omega(\rho) + I_N\otimes \sigma^2 I_d\right)^{-1},\\
%     \mu^* &= \left(K(z,Z_{t}^j)\otimes \Omega(\rho)\right)\Lambda \text{vec}\left((x^{t^j_k})\mid_{k=1:j_t}\right)\label{mu},\\
%     \Sigma^* &= K(z,z)\otimes\Omega(\rho) - \\ &\quad\left(K(z,Z_{t}^j)\otimes \Omega(\rho)\right)\Lambda\left(K(Z_{t}^j,z)\otimes \Omega(\rho)\right).
% \label{sigma}
% \end{align}
% The posterior predictive distribution of $X_{t+1}(z^{t+1}_{1:n^{(t+1)}})$ is then simply
% \begin{eqnarray}
%     & & X_{t+1}(z^{t+1}_{1:n^{(t+1)}})|x_{t^j_1}, \ldots, x_{t^j_{j_t}} \nonumber \\ &\sim& \mathcal{N}_{n^{(t+1)}d}\left(\mu^*, \Sigma^* + I_{n^{(t+1)}}\otimes \sigma^2 I_d\right)
% \end{eqnarray}
% with $z=z^{t+1}_{1:n^{(t+1)}}$ in Eq.~\eqref{eq:Posterior_mean_var}.
% \end{prop}

\subsection{Multivariate response case}
\subsubsection{Efficient matrix inversion for the multivariate case}
The target is to update the matrix 
\[
    \Lambda^{(t)} = \left(K(Z_k^t,Z_k^t)\otimes \Omega(\rho) + \sigma^2 I_{Nd}\right)^{-1}
\]
in an online manner.

Recall that this is required for the computing the posterior predictive distribution for the $k$th cluster at every time point but we only need to invert this afresh every time some new data is assigned to this cluster. We assume that the dimension of the data is fixed whereas the number of total data points increases each time new data is given.

Our approach is the following. At the first time some data is given to $\phi_k$, we estimate $\rho$ using the method discussed in the previous section. Denote this estimate as $\rho^{(1)}$. If $K^{(1)}$ represents the RBF kernel of the GP over $Z_k^1$, we set $\Lambda^{(1)}$ so that
\[
\Lambda^{(1)} = \left(K^{(1)}\otimes \Omega(\rho^{(1)}) + \sigma^2 I_{n_1d}\right)^{-1}.
\]

% Note that $d$ is fixed whereas the number of data is increasing.
% every time data is given and then keep the resulting matrix stored. 
% Note that the dimension of this matrix ($nd\times nd$) is growing each time new data is given where $n$ is the total data size at the current stage and $d$, dimension of $x$ is fixed.
% Consider the first time some data is given (in the particular cluster), $\{z^{(1)}_i,x^{(1)}_i\}_{i=1}^{n_1}$. First we estimate $\rho$ as discussed in the previous section, we call it $\rho^{(1)}$. Now we just compute $$\Lambda^{(1)} = \left(K(x^{(1)},x^{(1)})\otimes \Omega(\rho^{(1)}) + \sigma^2 I_{n_1d}\right)^{-1}$$
% matrix (note it requires inverting a $n_1d\times n_1d$ matrix) and save it.

Now suppose we assign new data to this cluster at times $t_1,\dots,t_{N_k^{(t)}}$ and each time we compute $\Lambda^{(1)}, \dots, \Lambda^{(N_k^{(t)})}$ (each time saving the last matrix). Now let at some later time point, $t'$, some new data is given to $\phi_k$. We have saved $\rho^{(N_k^{(t)})}$ and $\Lambda^{(N_k^{(t)})}$. Then, we again first use the new data to estimate $\rho^{(N_k^{(t)}+1)}$. Let $K(Z_t^k,Z_t^k)$ denote the RBF kernel with the old data till $t$, $K(z^{(t')},z^{(t')})$ represent the RBF kernel with the new data at $t'$, and $K(Z_t^k, z^{(t')})$ designate the RBF kernel computed between the old and new data. Also write $N=\sum_{s\leq N_k^{(t)}} n^{(t_s)}$ be the total number of data points associated with this cluster till time $t$. We can then write
\begin{equation}
\left(\Lambda^{(N_k^{(t)}+1)}\right)^{-1} = \begin{bmatrix} A_{11} & A_{12} \\ A_{21} & A_{22} \end{bmatrix},
\label{lambda}
\end{equation}
where
\begin{align*}
    A_{11} &= K(Z_t^k,Z_t^k)\otimes \Omega(\rho^{(N_k^{(t)} - 1)}) + \sigma^2 I_{N d},\\% \quad & (n_1d \times n_1 d)\\
    A_{12} &= K(Z_t^k,z^{(t')})\otimes \Omega(\rho^{(N_k^{(t)})}),\\% \quad& (n_1d \times n_2 d)\\
    A_{21} &= A_{12}^T,\\% \quad& (n_2d \times n_1 d)\\
    A_{22} &= K(z^{(t')},z^{(t')})\otimes \Omega(\rho^{(t')}) + \sigma^2 I_{n^{(t')} d}.\\% \quad& (n_2d \times n_2 d).
\end{align*}

% $\{z^{(2)}_i,x^{(2)}_i\}_{i=1}^{n_2}$ and we need to find $\Lambda^{(2)}$ for which we use the structure of the covariance matrix of $x$ based on our model. Let $x=(x^{(1)}, x^{(2)})$ and $z=(z^{(1)},z^{(2)})$ be the total data till now of size $n=n_1+n_2$. We estimate $\rho$ as before to get a new value, call it $\rho^{(2)}$. Typically we would compute the inverse of $\tilde{\Sigma}^{(2)} = K(z,z)\otimes \Omega(\rho^{(2)}) + \sigma^2 I_{nd}$ which is $nd\times nd$ covariance matrix of $x$, however we wish to use $\Lambda^{(1)}$ to ease the computation. We keep the $\rho^{(1)}$ where it was used (since we do not wish to redo the computation) and use the new $\rho^{(2)}$ for the other parts. In particular write 
% \begin{equation}
%     \Sigma^{(2)} = \begin{bmatrix} A_{11} & A_{12} \\ A_{21} & A_{22} \end{bmatrix}
% \end{equation}
% where
% \begin{align*}
%     A_{11} &= K(z^{(1)},z^{(1)})\otimes \Omega(\rho^{(1)}) + \sigma^2 I_{n_1 d} \quad& (n_1d \times n_1 d)\\
%     A_{12} &= K(z^{(1)},z^{(2)})\otimes \Omega(\rho^{(2)}) \quad& (n_1d \times n_2 d)\\
%     A_{21} &= A_{12}^T \quad& (n_2d \times n_1 d)\\
%     A_{22} &= K(z^{(2)},z^{(2)})\otimes \Omega(\rho^{(2)}) + \sigma^2 I_{n_2 d} \quad& (n_2d \times n_2 d)
% \end{align*}
Using the matrix inverse for block matrix, we have that
\begin{equation}
    \Lambda^{(N_k^{(t)}+1)} = %\left(\Sigma^{(2)}\right)^{-1} = 
    \begin{bmatrix}
    \Lambda^{(N_k^{(t)}+1)}_{11} & \Lambda^{(N_k^{(t)}+1)}_{12} \\
    \Lambda^{(N_k^{(t)}+1)}_{21} & \Lambda^{(N_k^{(t)}+1)}_{22}
    \end{bmatrix}
\end{equation}
with
\begin{align*}
\Lambda^{(N_k^{(t)}+1)}_{11} &= A_{11}^{-1} + A_{11}^{-1}A_{12}\Lambda^{(N_k^{(t)}+1)}_{22}A_{21}A_{11}^{-1}, \\
\Lambda^{(N_k^{(t)}+1)}_{12} &= -A_{11}^{-1}A_{12}\Lambda^{(N_k^{(t)}+1)}_{22}, \\
\Lambda^{(N_k^{(t)}+1)}_{21} &= -\Lambda^{(2)}_{22}A_{21}A_{11}^{-1},\\
\Lambda^{(N_k^{(t)}+1)}_{22} &= \left(A_{22} - A_{21}A_{11}^{-1}A_{12}\right)^{-1}.
\end{align*}
% where 
% $$C = \left(A_{22} - A_{21}A_{11}^{-1}A_{12}\right)^{-1}$$
While we now need two inverses to compute $\Lambda^{(N_k^{(t)}+1)}$, updating it in this manner is computationally faster. We have already calculated $A_{11}^{-1}$ because $A_{11}^{-1} = \Lambda^{(N_k^{(t)})}$. Further, the other inverse is the inverse of a $n^{(t')}d \times n^{(t')}d$ matrix. As more data is added to a GP, it will be much faster to invert such a matrix compared to the matrix in \eqref{lambda} as a whole which has dimensions of  $(N+n^{(t')})d\times (N+n^{(t')})d$.

\subsubsection{Estimating $\rho$}
We apply a moment-matching approach to find a suitable plug-in estimator of $\rho$ for $\phi_j$ at time $t$. If $\bar{X}_t = \frac{1}{N^{(t)}_{k}} \sum_{t' = t_1, t_2, \ldots, t_{T_k}} X_{t'}$, our goal is to match the second moment from the MRGP assumption to $\frac{1}{Nd} (X_t - \bar{X}_t) (X_t-\bar{X}_t)^T$. Due to the form of $\Omega(\rho)$, we have after some rearrangement that $\rho A = B$ where
\begin{align*}
    A &= \left[ K\otimes \mathbf{1}_d\mathbf{1}_d^T - K \otimes I_d\right] \\
    B &= \frac{1}{N d}(X_t-\bar{X}_t) (X_t-\bar{X}_t)^T - \sigma^2 I_{Nd} - K \otimes I_d.
\end{align*}

Here $K=K(Z_t^j,Z_t^j)$ is the $N\times N$ covariance kernel over all the spatial locations associated to cluster $j$ till time $t$. We then choose our estimator, $\widehat{\rho}$, so as to minimize $\norm{\rho A - B}_F$. This has the following closed form solution:
\begin{equation}
    \widehat{\rho} = \frac{\sum_{i,j} A_{ij}B_{ij}}{\sum_{i,j} A_{ij}^2}
\end{equation}

\section{Experiments}
This section contains additional experimental results. 
Figure~\ref{fig:heat20} provides an illustration of the transition matrix for the NGSIM data. The transition matrix reflects the natural intuition that traffic flows tend to stay in the same state over a short duration of time. This is reflective of the fact that traffic lights may control the flow of traffic for a certain duration of time and then as traffic light directions change, so does the traffic flow pattern. On the other hand, movement patterns are quite flexible over large durations as shown by the 60-step transition matrix.

Figure~\ref{fig:clusters} shows the 8 true functions(on the left) that were used for the purpose of simulations along with their data estimates (on the right). By simple eyeballing, the estimates look to match the true functions closely.

Table~\ref{table:data} provides results with NGSIM data for various parameter settings with infinite HMM-GP.

\begin{table}[h!]
\centering
\begin{tabular}{|| c | c c c c||} 
 \hline
 $\sigma_0=\ell_0$ & $\sigma$ & $K$ & log-lik & time \\ [0.5ex] 
 \hline\hline
 \multirow{5}{*}{0.07} & 0.1 & 109 & -31487.76 & 5299.7 \\ 
  & 0.15 & 64 & -13808.98 & 7093.2 \\
  & 0.2 & 41 & -15939.11 & 10885.7 \\
  & 0.25 & 29 & -21262.44 & 18023.6 \\
  & 0.3 & 20 & -27805.92 & 22157.2 \\ [1ex] 
 \hline
 \multirow{5}{*}{\textbf{0.1}} & 0.1 & 166 & -26461.12 & 3848.9 \\ 
  & \textbf{0.15} & \textbf{99} & \textbf{-2790.70} & \textbf{4467.4} \\
  & 0.2 & 66 & -3878.23 & 5254.7  \\
  & 0.25 & 41 & -12989.42 &  7619.7\\
  & 0.3 & 32 & -23212.65 & 9956.6 \\ [1ex] 
 \hline
\multirow{5}{*}{0.3} & 0.1 & 343 & -127120.84 & 5112.8 \\ 
& 0.15 & 213 & -42825.24 & 4734.9 \\
& 0.2 & 138 & -25478.42 & 5775.6 \\
& 0.25 & 96 & -23204.13 & 6972.8 \\
& 0.3 & 74 & -28864.42 & 8752.1 \\ [1ex]
\hline
\end{tabular}
\caption{Table showing performance of our algorithm on the NGSIM data. The hyperparameters with the highest log likelihood is bolded.}
\label{table:data}
\end{table}

\begin{figure}
    \centering
    \centering     %%% not \center
    \subfigure[]{\label{fig:a}\includegraphics[width=40mm]{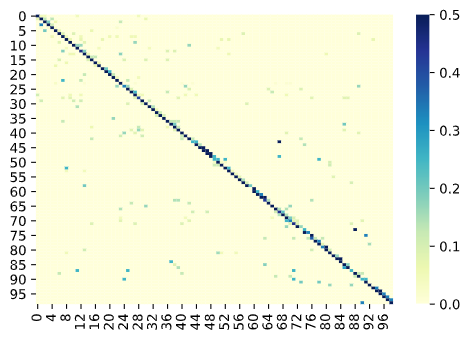}}
    \subfigure[]{\label{fig:b}\includegraphics[width=40mm]{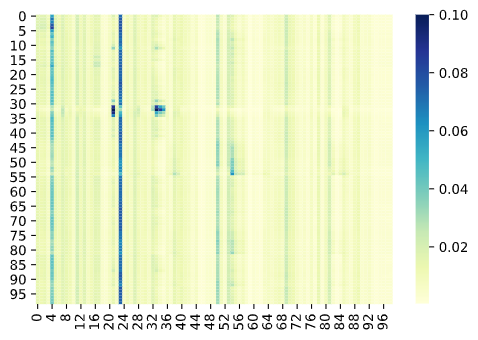}}
    \caption{Heatmaps showing the estimated 1-step transition matrix for all hidden states in (a) and 60-step in (b)}
    \label{fig:heat20}
\end{figure}

\begin{figure}
\centering
\subfigure[]{\label{fig:a}\includegraphics[width=40mm]{Plots/true1.png}}
\subfigure[]{\label{fig:b}\includegraphics[width=40mm]{Plots/est0.png}}

\subfigure[]{\label{fig:a}\includegraphics[width=40mm]{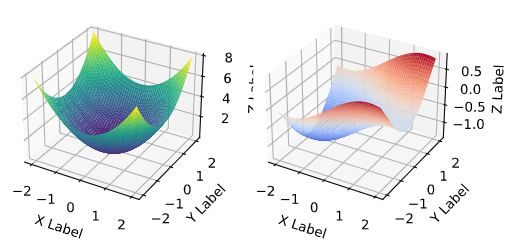}}
\subfigure[]{\label{fig:b}\includegraphics[width=40mm]{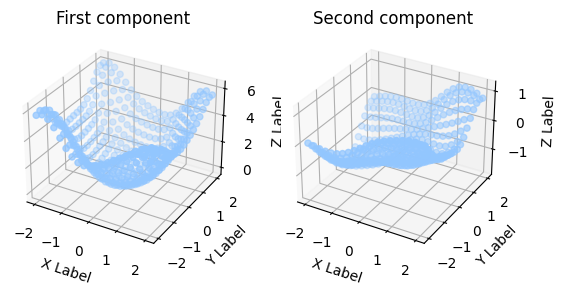}}

\subfigure[]{\label{fig:a}\includegraphics[width=40mm]{Plots/true3.png}}
\subfigure[]{\label{fig:b}\includegraphics[width=40mm]{Plots/est3.png}}

\subfigure[]{\label{fig:a}\includegraphics[width=40mm]{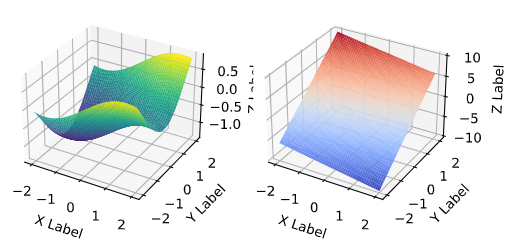}}
\subfigure[]{\label{fig:b}\includegraphics[width=40mm]{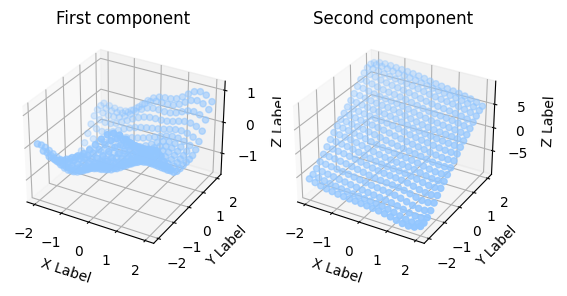}}

\subfigure[]{\label{fig:a}\includegraphics[width=40mm]{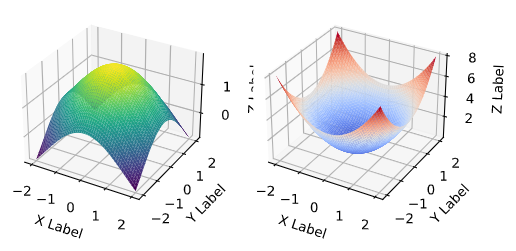}}
\subfigure[]{\label{fig:b}\includegraphics[width=40mm]{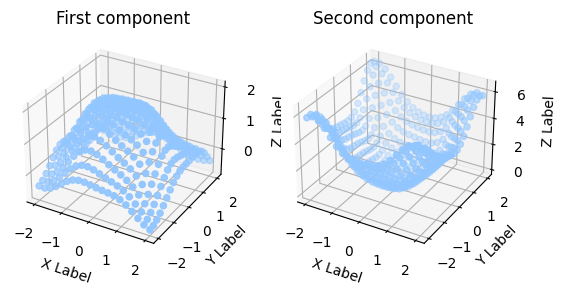}}

\subfigure[]{\label{fig:a}\includegraphics[width=40mm]{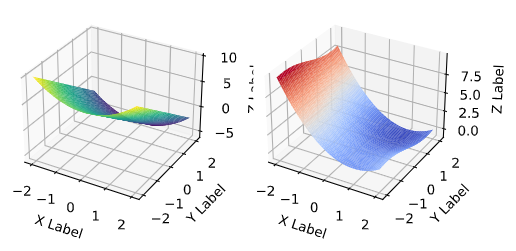}}
\subfigure[]{\label{fig:b}\includegraphics[width=40mm]{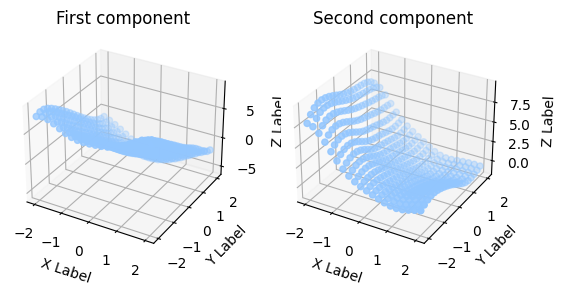}}

\subfigure[]{\label{fig:a}\includegraphics[width=40mm]{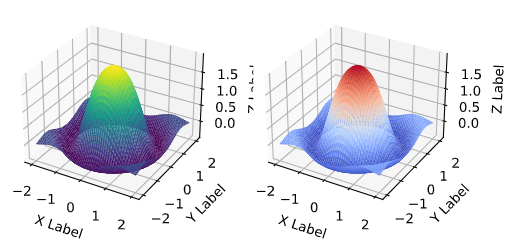}}
\subfigure[]{\label{fig:b}\includegraphics[width=40mm]{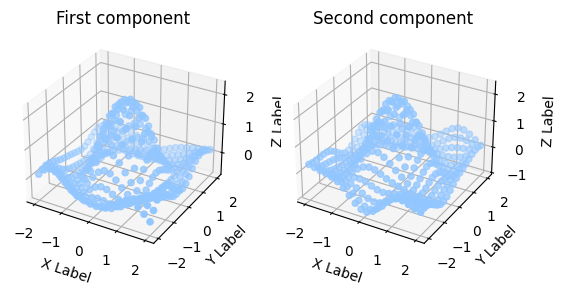}}

\subfigure[]{\label{fig:a}\includegraphics[width=40mm]{Plots/true8.png}}
\subfigure[]{\label{fig:b}\includegraphics[width=40mm]{Plots/est2.png}}
\caption{Simulation data: Left column shows the true 8 functions (each from $[-2,2]\times [-2,2]\to \mathbb{R}^2)$ while the right column gives the 8 estimated clusters}
\label{fig:clusters}
\end{figure}

\end{document}